%% file: main.tex
\ifx\mycmd\undefined
\documentclass[lettersize,journal]{IEEEtran}

\usepackage{amsmath,amssymb,amsthm,amsfonts,mathtools} 
\usepackage{dsfont,eucal,bbm,bm,nicefrac} 
\usepackage{graphicx,float,subcaption,booktabs} 
	\graphicspath{{./figures/}}
\usepackage{algorithm,algorithmic} 
\usepackage{hyperref} 
\usepackage{tikz,xcolor,soul} 
\usepackage{cite}
\usepackage{textcomp}

\newtheorem{theorem}{Theorem}
\newtheorem{corollary}{Corollary}[theorem]
\newtheorem{lemma}[theorem]{Lemma}

\newtheorem{remark}{Remark}

\newtheorem{definition}{Definition}

\newcommand{\T}{\ensuremath{\mathrm{T}}} 

\newcommand{\norm}[1]{\ensuremath{\left\| #1\right\|}}
\newcommand{\pbra}[1]{\ensuremath{\left( #1\right)}}
\newcommand{\sbra}[1]{\ensuremath{\left[ #1\right]}}
\newcommand{\cbra}[1]{\ensuremath{\left\{ #1\right\}}}
\newcommand{\abra}[1]{\ensuremath{\left< #1\right>}}

\newcommand{\pder}[2]{\ensuremath{\frac{\partial #1}{\partial #2}}}
\newcommand{\E}[1]{\ensuremath{\mathbb{E}\left[ #1\right]}}

\definecolor{highlight}{rgb}{0.9,0.9,0.9}
\sethlcolor{highlight}

\DeclareMathOperator*{\argmax}{arg\,max}
\DeclareMathOperator*{\argmin}{arg\,min}

\else 
\documentclass[review,onefignum,onetabnum]{siamonline220329}

\input{macros}

\ifpdf
\hypersetup{
  pdftitle={Multi-Resolution Online Deterministic Annealing},
  pdfauthor={C. N. Mavridis, and J. S. Baras}
}
\fi

\fi

\begin{document}

\ifx\mycmd\undefined
\title{Multi-Resolution Online Deterministic Annealing: A Hierarchical and Progressive Learning Architecture}

\author{Christos N. Mavridis, \IEEEmembership{Member, IEEE}, and 
John S. Baras, \IEEEmembership{Life Fellow, IEEE}
\thanks{
The authors are with the 
Department of Electrical and Computer Engineering and 
the Institute for Systems Research, 
University of Maryland, College Park, USA.
{\tt\small emails:\{mavridis, baras\}@umd.edu}.}%
\thanks{Research partially supported by the 
Defense Advanced Research Projects
Agency (DARPA) under Agreement No. HR00111990027, 
by ONR grant N00014-17-1-2622, 
and by a grant from Northrop Grumman Corporation.%
}%
}
\fi

\maketitle

\ifx\mycmd\undefined
 \thispagestyle{empty}
\pagestyle{empty}
\fi


\begin{abstract}
Hierarchical learning algorithms that gradually approximate a solution to a data-driven optimization problem
are essential to decision-making systems, especially under limitations on time and computational resources.
In this study, we introduce a general-purpose hierarchical learning architecture
that is based on the progressive partitioning of a possibly multi-resolution data space.
The optimal partition is gradually approximated by solving a sequence 
of optimization sub-problems 
online, using 
gradient-free stochastic approximation updates.
As a consequence, a function approximation problem can be defined within each subset of the partition
and solved using the theory of two-timescale stochastic approximation.
This simulates an annealing process and defines a robust and interpretable heuristic method 
to gradually increase the complexity of the learning architecture
in a task-agnostic manner, 
giving emphasis to regions of the data space that are considered more important according to 
a predefined criterion.
Finally, by imposing a tree structure in the progression of the partitions, 
we provide a means to incorporate potential multi-resolution structure of the data space into this approach,
significantly reducing its complexity, while introducing hierarchical 
variable-rate feature extraction properties similar 
to certain classes of deep learning architectures. 
Asymptotic convergence analysis and experimental results are provided
for supervised and unsupervised learning problems.
\end{abstract}


\ifx\mycmd\undefined


\else



\fi

\section{Introduction}
\label{Sec:Introduction}

\ifx\mycmd\undefined
\IEEEPARstart{L}{earning}
\else
Learning
\fi
from observations is pivotal to autonomous decision-making and communication systems. 
Mathematically, such learning problems are often formulated as 
constrained stochastic optimization problems:
given realizations of a random variable $X\in S$ representing the observations, 
an optimal parameter vector $\theta\in\Theta$ is to be found such that
a well-defined error measure between an unknown function $f(X)\in\mathcal{F}$
and a learning model $\hat f(X,\theta)\in\mathcal{F}$, parameterized by $\theta$, 
is minimized under potentially additional constraints.
However, the solution of such problems over the entire domain $S$  
often requires the learning model $\hat f(X,\theta)$ to be particularly complex, 
making the estimation of $\theta$ costly, 
and raising issues with respect to phenomena such as over-fitting, generalization, and robustness, 
connected by an underlying trade-off between 
complexity and performance \cite{bennett2006interplay}.
As a result, the ability to gradually approximate a solution to these problems 
is essential to decision-making systems that often operate in real-time and under limitations in memory and 
computational resources.
%

%

Current deep learning methods 
have made progress towards the construction of a hierarchical representation of the data space
\cite{lecun2015deep,hinton2006fast,	krizhevsky2012imagenet,lee2009convolutional}.
However, such approaches do not necessarily satisfy the above description of hierarchical learning, 
since they typically use overly complex models over the entire data space $S$,
which comes in the expense of time, energy, data, memory, 
and computational resources \cite{thompson2020computational,strubell2019energy}.
In this work, we are mainly focusing on a framework for 
hierarchical progressive learning and data representation, 
where a gradually growing and hierarchically structured set of learning models is used 
for function approximation.
We consider a prototype-based learning framework where, given random observations of $X\in S$, a set of 
prototypes $\cbra{\mu_i}\in S$ (also called codevectors or neurons) are scattered in the data space $S$ to encode
subsets/regions $\cbra{S_i}$ that form a partition of $S$ \cite{biehl2016prototype}.
This adheres to the principles of vector quantization for signal compression \cite{Kohonen1995}.
In this regard, a knowledge representation can be defined as the set of codevectors $\cbra{\mu_i\in S}$ that 
induce a structured partition $\cbra{S_i}$ of the data space $S$, along with a set of local learning models 
$\hat f(x,\theta_i)$ associated with each region $S_i$, parameterized by their own set of parameters $\theta_i$.
A structured representation like this allows, among other things, to locate specific regions of the space
that the algorithm needs to approximate in greater detail, according to the problem at hand and the designer's requirements. 
This results in adaptively allocating more resources only in the subsets of the data space that are needed, and provides
benefits in terms of time, memory, and model complexity.
Moreover, learning with local models that take advantage of the differences in the underlying distribution of the data space 
provides a means to understand certain properties of the data space itself, 
i.e., this is an interpretable learning approach \cite{ruping2005learning}.
An illustration of this framework is given in Fig. \ref{fig:regression-diagram}.
\begin{figure*}[ht]
\centering
\begin{subfigure}[b]{0.28\textwidth}
\centering
\includegraphics[trim=0 0 0 0,clip,width=1.0\textwidth]{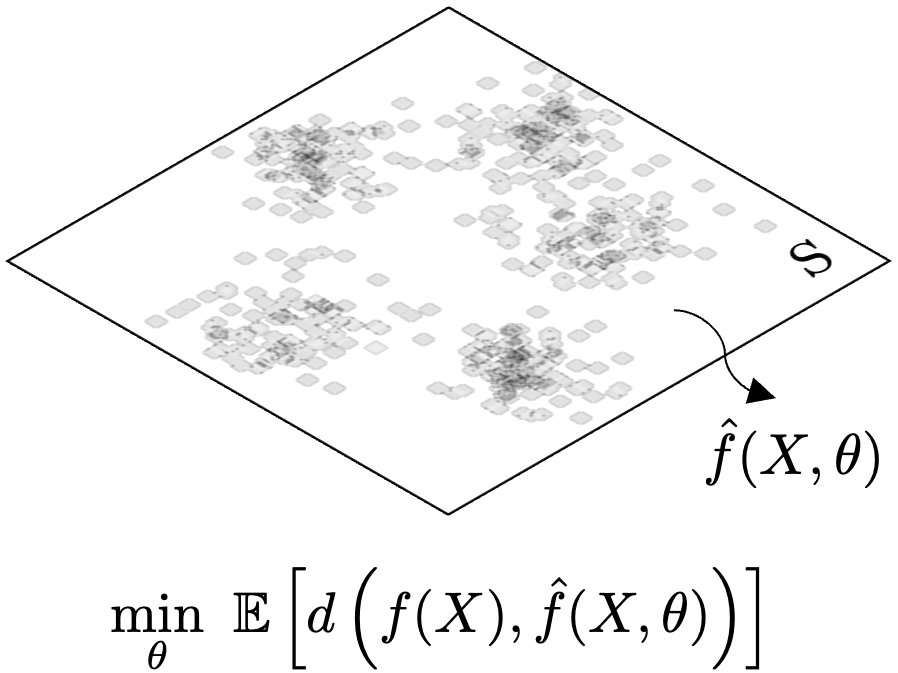}
\caption{Classical regression problem.  \phantom{asdf} \phantom{asdf} \phantom{asdf} \phantom{asdf}}
\end{subfigure}
\begin{subfigure}[b]{0.36\textwidth}
\centering
\includegraphics[trim=0 0 0 0,clip,width=1.0\textwidth]{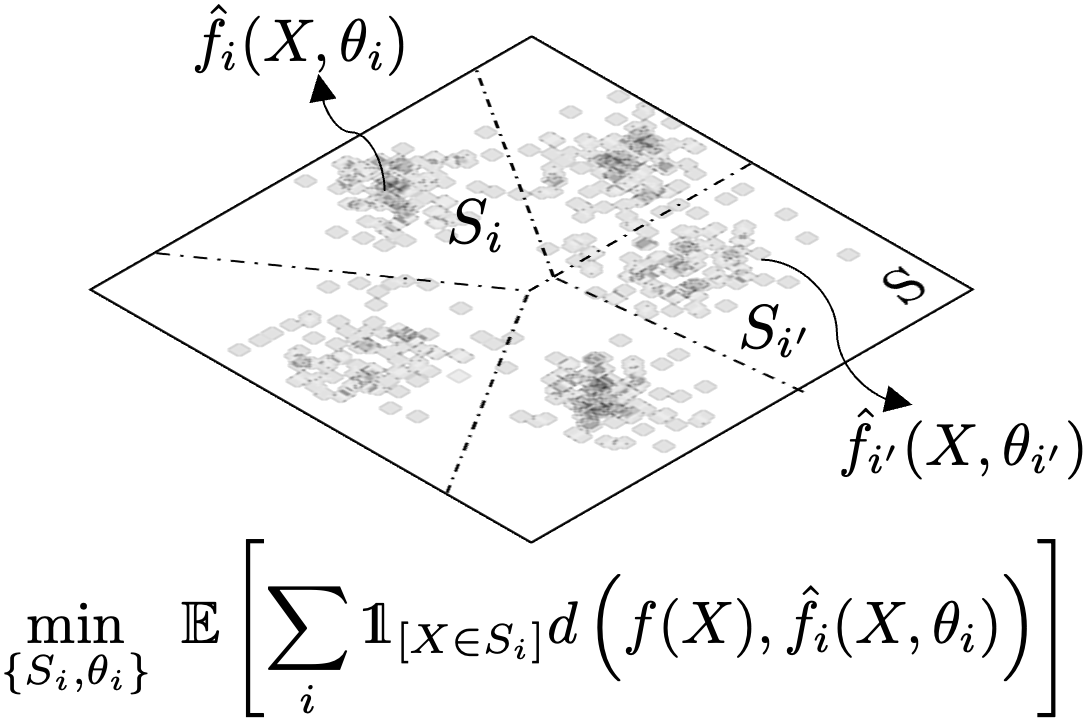}
\caption{Combined problem of partitioning and function approximation.}
\end{subfigure}
\begin{subfigure}[b]{0.33\textwidth}
\centering
\includegraphics[trim=0 0 0 0,clip,width=1.0\textwidth]{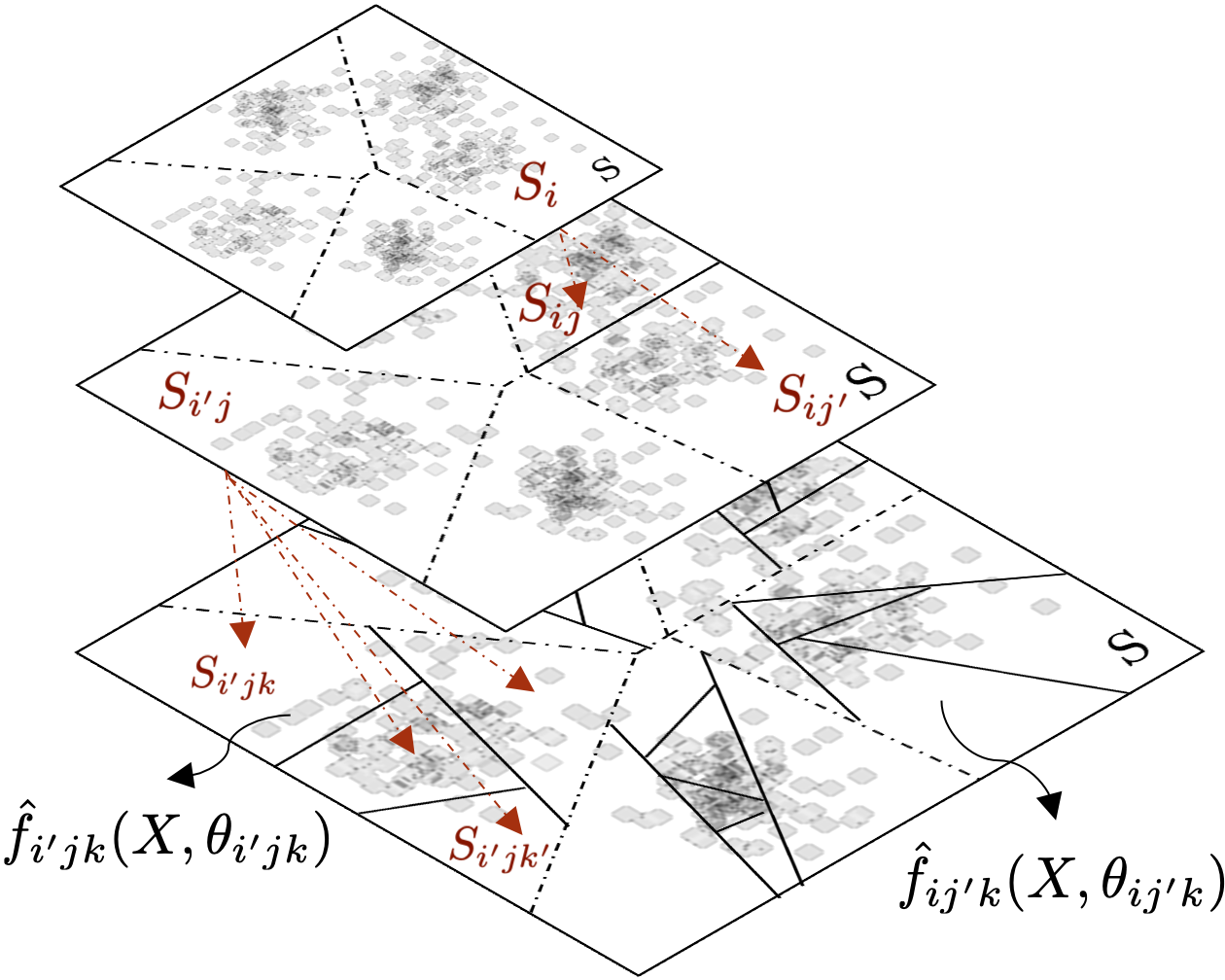}
\caption{Tree-structured partitioning and function approximation.}
\label{sfig:tree-regression}
\end{subfigure}
\caption{Comparison of the classical regression problem over the entire domain $S$
with the problem of combined partitioning and regression
within each subset of the partition.
Here the input $x\in S$ is a random variable and the function $f(x)$ is to be estimated over $S$ by 
(a) a single learning model $\hat f(x,\theta)$, and (b)-(c) a set of $\cbra{\hat f(x,\theta_i)}$ defined in each region $S_i$, 
where $\cbra{S_i}$  is a partition of S to be estimated as well.}
\label{fig:regression-diagram}
\end{figure*}

Regarding the learning process, we are interested in algorithms that are able 
to simultaneously solve both the problems of partitioning and function approximation,
given online (e.g., real-time) observations.
This is of great importance in many applications, and especially 
in the scope of learning algorithms for inference and control in general cyber-physical systems
\cite{mavridis2022annealing,mavridis2022sparse,mavridis2021progressive}.
To construct a sequence of partitions with increasing number of subsets
we build upon the notion of Online Deterministic Annealing \cite{mavridis2022online}
and define a series of
soft-clustering optimization problems:
\begin{align*}
\min_{\cbra{\mu_i}} ~ F_\lambda(X,Q):= (1-\lambda) D(X,Q) - \lambda H(X,Q), 
\end{align*}
%
parameterized by a Lagrange coefficient $\lambda\in[0,1]$
controlling the trade-off between minimizing 
an average distortion measure $D(X,Q) := \E{d(X,Q)}$,
for an appropriately defined dissimilarity measure $d$,
and maximizing the Shannon entropy $H(X,Q)$, with $H(X,Q):=\E{-\log p(X,Q)}$.
The novelty of the approach lies in the introduction of $Q$ as a random variable 
described by the association probabilities $p(\mu_i|X=x)$
that represents the probability of a data point $x$ to belong to the subset $S_i := \cbra{x\in S: i=\argmin_j d(x,\mu_j)}$.
Once the joint probability space of $(X,Q)$ is defined, 
successively solving the optimization problems $\min_{\cbra{\mu_i}} ~ F_\lambda(X,Q)$ 
for decreasing values of $\lambda$, leads in a series of 
bifurcation phenomena when the cardinality of the set of codevectors $\cbra{\mu_i}$ increases,
resembling an annealing process 
that introduces inherent robustness and regularization properties \cite{mavridis2022online,rose1998deterministic}.

An important property of this approach, initially shown in \cite{mavridis2022online}, 
is that the optimization problems $\min_{\cbra{\mu_i}} ~ F_\lambda(X,Q)$ can be solved online, 
using gradient-free stochastic approximation updates \cite{borkar2009stochastic}, 
as long as the measure $d$ belongs to the family of Bregman divergences, 
information-theoretic dissimilarity measures that include, among others, 
the widely used squared Euclidean distance and Kullback-Leibler divergence
\cite{banerjee2005clustering,villmann_onlineDLVQmath_2010}.
We exploit the fact that a stochastic approximation algorithm can be used 
as a training rule for constructing the partition $\cbra{S_i}$, 
to build a framework that simultaneously trains the learning models  
$\cbra{\hat f(x,\theta_i)}$ defined in each region $S_i$.
In particular, according to the theory of two-timescale stochastic approximation \cite{borkar2009stochastic},
we define two stochastic approximation algorithms that run at the same time and with the same observations
but with different stepsize schedules that define a fast and a slow learing process.  
In our case the slow process approximates the parameters $\cbra{\mu_i}$ and as a result the partition $\cbra{S_i}$, and
the fast process executes a function approximation algorithm within each $S_i$ to find the optimal parameters $\theta_i$
for the learning model $\hat f(x,\theta_i)$.

Finally, we further extend this approach by incorporating 
structural constraints in the construction of the partition $\cbra{S_i}$.
In particular, by imposing a non-binary tree structure in the growing set of the parameters $\cbra{\mu_i}$,
we show that we can both (a) greatly reduce the quadratic (in the number of parameters $\mu_i$) complexity of the approach, and
(b) construct a hierarchical and progressively growing tree-structured partition 
where each layer of the tree is trained using  
different resolution representation of the data space,
according to an independent multi-resolution analysis.
While this is a general framework for multi-resolution learning, we show that, 
in the case when convolution-based multi-resolution features are used, 
the proposed architecture shares similarities with deep learning approaches such as 
Deep Convolutional Networks \cite{lecun2015deep} and
Scattering Convolutional Networks \cite{bruna2013invariant}. 
Lastly, we provide asymptotic convergence analysis of the proposed learning architecture 
and experimental results to illustrate its properties in 
clustering, classification, and regression applications.

The paper is organized as follows:
Section \ref{Sec:ODA} introduces the Online Deterministic Annealing framework for progressive partitioning
along with a mathematical analysis of its properties. 
Section \ref{Sec:ODA-Regression} develops the two-timescale framework for combined partitioning and function approximation.
Section \ref{Sec:ODA-Classification} handles the problem of classification in two different approaches.
Section \ref{Sec:MR-ODA} extends the general model by incorporating tree-structure constraints and 
multi-resolution representation of the data space.
Finally, Section \ref{Sec:Results} illustrates experimental results, and
Section \ref{Sec:Conclusion} concludes the paper.

\section{Online Deterministic Annealing for Progressive Partitioning}
\label{Sec:ODA}

In this section we provide a comprehensive review of 
the online deterministic annealing approach introduced in 
\cite{mavridis2022online} and \cite{mavridis2022annealing},
as well as additional analytical results and insights that will be used 
in Sections \ref{Sec:ODA-Regression}, \ref{Sec:ODA-Classification}, and \ref{Sec:MR-ODA},
to construct the proposed hierarchical learning architecture.

We start our analysis with the case of unsupervised learning, where partitioning a space $S$
is equivalent to the problem of clustering and density estimation. 
In this context, the observations (data) 
are independent realization of a random variable $X: \Omega \rightarrow S$ defined in a 
probability space $\pbra{\Omega, \mathcal{F}, \mathbb{P}}$, 
where $S\subseteq \mathbb{R}^d$ is the observation space (data space).
In a prototype-based learning approach one defines a 
similarity measure $d:S\rightarrow ri(S)$,
and a set of $K$ prototypes/codevectors $\mu := \cbra{\mu_i}_{i = 1}^K$, $\mu_i \in ri(S)$, 
that define a partition $\cbra{S_i: x\in S: i=\argmin_j d(x,\mu_j)}$ such that 
the following average distortion measure is minimized: 
\begin{equation}
    \min_{\mu} ~ J(\mu) := \E{\sum_i \mathds{1}_{\sbra{X\in S_i}} d(X,\mu_i)}
    \label{eq:VQ}
\end{equation}
Here $ri(S)$ represents the relative interior of $S$,
and $\mathds{1}_{A}$ is the indication function of an event $A$.
The similarity measure as well as the number of prototypes $K$
are predefined designer parameters. 
This process is equivalent to finding the most suitable model out of a set of $K$
local constant models, and results in a piecewise-constant approximation of 
the data space $S$.
This representation has been used for clustering in vector quantization 
applications \cite{mavridis2020convergence, Kohonen1995}, 
and, in the limit $K\rightarrow\infty$, can be used for density estimation.
%

%
%

To construct a method that progressively increases the number of prototypes $K$,
we adopt a probabilistic approach similar to \cite{mavridis2022annealing,mavridis2022online}, and define
a discrete random variable $Q:S \rightarrow ri(S)$ 
such that \eqref{eq:VQ} takes the form
\begin{equation}
\begin{aligned}
\min_{\mu} ~ D(\mu) &:= \E{d\pbra{X,Q}} 
\\& 
=\E{\E{d(X,Q)|X}} 
\\&
=\int p(x) \sum_i p(\mu_i|x) d(x,\mu_i) ~dx
\end{aligned}
\label{eq:softVQ}
\end{equation}	
Notice that $Q$ is completely described by the association probabilities 
$\{p(\mu_i|x):=\mathbb{P}[Q=\mu_i|X=x] \}$, $\forall i$.
This is now a problem of finding both the locations $\cbra{\mu_i}$ and the 
association probabilities $\cbra{p(\mu_i|x)}$.
Therefore this is a more general problem than (\ref{eq:VQ}),
where it is subtly assumed that 
$p(\mu_i|x)=\mathds{1}_{\sbra{x \in S_i}}$.

%
The definition of the random variable $Q$ allows us to constraint 
the distribution of $(X,Q)$ by maximizing the entropy:
\begin{equation}
\begin{aligned}
H(\mu) &:= \E{-\log P(X,Q)} 
=H(X) + H(Q|X)
\\&
=H(X) - \int p(x) \sum_i p(\mu_i|x) \log p(\mu_i|x) ~dx,
\end{aligned}    
\end{equation}
at different levels.
This is essentially a realization of the 
Jaynes's maximum entropy principle \cite{jaynes1957information}.
%
We formulate this multi-objective optimization as the minimization of the Lagrangian
\begin{equation}
\min_\mu F_\lambda(\mu) := (1-\lambda) D(\mu) - \lambda H(\mu)
\label{eq:F}
\end{equation}
where $\lambda\in[0,1)$ acts as a Lagrange multiplier.
The term 
%
$ T:=\frac{\lambda}{1-\lambda} ,\ \lambda\in[0,1) $
%
can be seen as a temperature coefficient 
in a deterministic annealing process \cite{mavridis2022online}.
In this regard, this approach follows from the Online Deterministic Annealing (ODA) algorithm in \cite{mavridis2022online}, 
and its offline predecessor \cite{rose1998deterministic}.
Equation \eqref{eq:F} represents the scalarization method for trade-off 
analysis between two performance metrics, one related to performance, and 
one to generalization.
The entropy $H$, 
acts as a regularization term, and is given progressively less weight 
as $\lambda$ (resp. $T$) decreases.
For large values of $\lambda\rightarrow 1$ (resp. $T\rightarrow \infty$) we essentially maximize the entropy, and
as $\lambda$ (resp. $T$) is lowered, we transition 
from one Pareto point to another in a naturally occurring direction
that resembles an annealing process.

In the remaining section, we will (i) derive an analytical solution of the optimization problem (\ref{eq:F})
and a recursive gradient-free training rule to approximate it online, 
(ii) show that the number of unique locations $\cbra{\mu_i}$ is finite for $\lambda>0$ and increases
as $\lambda$ decreases beyond certain critical values with respect to a bifurcation phenomenon, and 
(iii) analyze the asymptotic behavior and complexity of this approach.

\subsection{Solving the Optimization Problem}

As in the case of standard vector quantization algorithms, 
we will minimize $F_\lambda$ in \eqref{eq:F} 
by successively minimizing it first respect to the 
association probabilities $\cbra{p(\mu_i|x)}$, 
and then with respect to the codevector locations $\mu$.
The following lemma provides the solution of 
minimizing $F_\lambda$ with respect to the association probabilities $p(\mu_i|x)$: 
\begin{lemma}
The solution of the optimization problem
\begin{equation}
\begin{aligned}
F_\lambda^*(\mu) &:= \min_{\cbra{p(\mu_i|x)}} F_\lambda(\mu)
\\
\text{s.t.} & \sum_i p(\mu_i|x) = 1 
\end{aligned}
\label{eq:Fstar}
\end{equation}
is given by the Gibbs distributions
\begin{equation}
p^*(\mu_i|x) = \frac{e^{-\frac{1-\lambda}{\lambda}d(x,\mu_i)}}
			{\sum_j e^{-\frac{1-\lambda}{\lambda}d(x,\mu_j)}} ,~ \forall x\in S
\label{eq:gibbs}
\end{equation}
\label{lem:gibbs}
\end{lemma}
\begin{proof}
See Appendix \ref{App:gibbs}.
\end{proof}

In order to minimize $F_\lambda^*(\mu)$ with respect to the codevector locations $\mu$ 
we observe that
\begin{align*}
&\frac d {d\mu} F_\lambda^*(\mu)= 
\int p(x) \sum_i (1-\lambda) \frac d {d\mu} \pbra{p^*(\mu_i|x) d_\phi(x,\mu_i)} 
\\& \quad\quad\quad
    + \lambda \frac d {d\mu} \pbra{p^*(\mu_i|x) \log p^*(\mu_i|x)} ~dx 
\\&\phantom{\frac d {d\mu} F_\lambda^*(\mu)}=
\int p(x) \sum_i (1-\lambda) \frac d {d\mu} p^*(\mu_i|x) d_\phi(x,\mu_i)
\\& \quad\quad\quad   
    + (1-\lambda) p^*(\mu_i|x) \frac d {d\mu} d_\phi(x,\mu_i) 
    +\lambda \frac d {d\mu} p^*(\mu_i|x) 
\\& \quad\quad\quad    
    + \lambda \frac d {d\mu} p^*(\mu_i|x) \log p^*(\mu_i|x) ~dx 
\\&\phantom{\frac d {d\mu} F_\lambda^*(\mu)}=
\int p(x) \sum_i (1-\lambda) p^*(\mu_i|x) \frac d {d\mu} d_\phi(x,\mu_i) 
\\&     
+\lambda \frac d {d\mu} p^*(\mu_i|x) 
    - \lambda \frac d {d\mu} p^*(\mu_i|x) \sum_j e^{-\frac{1-\lambda}{\lambda} d_\phi(x,\mu_j)} ~dx 
\\&\phantom{\frac d {d\mu} F_\lambda^*(\mu)}=
\sum_i \int p(x) p^*(\mu_i|x) \frac d {d\mu_i} d(x,\mu_i) ~dx 
\end{align*}
such that 
\begin{equation}
\frac d {d\mu} F_\lambda^*(\mu) =0 \implies 
\sum_i \int p(x) p^*(\mu_i|x) \frac d {d\mu_i} d(x,\mu_i) ~dx = 0
\label{eq:M}
\end{equation}
where we have used (\ref{eq:gibbs}), direct differentiation, and 
$\sum_i \frac d {d\mu} p^*(\mu_i|x) = \frac d {d\mu} \sum_i p^*(\mu_i|x) = 0$.
In the following section, we show that 
(\ref{eq:M}) has an easy to compute 
closed form solution if the dissimilarity measure $d$ 
belongs to the family of Bregman divergences.

\subsection{Bregman Divergences as Dissimilarity Measures}

The proximity measure $d$ can be generalized 
to dissimilarity measures inspired by 
information theory and statistical analysis.
In particular, 
the family of Bregman divergences
can offer numerous advantages in learning applications
compared to the Euclidean distance alone \cite{banerjee2005clustering}. 
\begin{definition}[Bregman Divergence]
	Let $ \phi: S \rightarrow \mathbb{R}$, 
	be a strictly convex function defined on 
	a vector space $S\subseteq \mathbb{R}^d$ 
	such that $\phi$ 
	is twice F-differentiable on $S$. 
	The Bregman divergence 
	$d_{\phi}:H \times S \rightarrow \left[0,\infty\right)$
	is defined as:
	\begin{align*}
		d_{\phi} \pbra{x, \mu} = \phi \pbra{x} - \phi \pbra{\mu} 
							- \pder{\phi}{\mu} \pbra{\mu} \pbra{x-\mu},
	\end{align*}
	where $x,\mu\in S$, and the continuous linear map 
	$\pder{\phi}{\mu} \pbra{\mu}: S \rightarrow \mathbb{R}$ 
	is the Fr\'echet derivative of $\phi$ at $\mu$.
	\label{def:BregmanD}
\end{definition}
The derivative of $d_\phi$ with respect to the second argument 
can be written as
\begin{align}
\pder{d_{\phi}}{\mu}(x,\mu) 
= - \pder{^2 \phi(\mu)}{\mu^2}(x-\mu) 
= - \abra{\nabla^2 \phi(\mu),(x-\mu)}
\label{eq:dd_phi}
\end{align}
which leads to the following theorem showing that 
if $d$ is a Bregman divergence, 
the solution to the second optimization step 
\eqref{eq:M} can be analytically computed in a convenient centroid form:
\begin{theorem}
A sufficient condition for the solution of the optimization problem
\begin{equation}
    \min_\mu F_\lambda^*(\mu)
    \label{eq:minFstar}
\end{equation}
where $F_\lambda^*(\mu)$ is defined in \eqref{eq:Fstar}, is given by 
\begin{equation}
\mu_i^* = \E{X|\mu_i} = \frac{\int x p(x) p^*(\mu_i|x) ~dx}{p^*(\mu_i)}
\label{eq:mu_star}
\end{equation}
if $d:=d_\phi$ is a Bregman divergence for some function 
$\phi$ that satisfies Definition \ref{def:BregmanD}.
\label{thm:bregman_in_DA}
\end{theorem}
\begin{proof}
Given \eqref{eq:dd_phi}, \eqref{eq:M} becomes
\begin{equation}
\int (x-\mu_i) p(x) p^*(\mu_i|x) ~dx= 0
\end{equation}
which is equivalent to (\ref{eq:mu_star}) since 
$\int p(x) p^*(\mu_i|x) ~dx = p^*(\mu_i)$.
\end{proof}

As a final note, the family of Bregman divergences includes 
two notable examples. The first is the widely used squared Euclidean distance 
$d_\phi(x,\mu) = \|x-\mu\|^2$ ($\phi(x) = \abra{x,x},\ x\in\mathbb{R}^d$),
and the second is the generalized Kullback-Leibler divergence 
$d_\phi(x,\mu) = \abra{x,\log x - \log \mu}
	- \abra{\mathds{1}, x - \mu}$
($\phi(x) = \abra{x,\log x},\ x\in\mathbb{R}_{++}^d$).

\subsection{The Online Learning Rule}

In an offline approach, the approximation of 
the conditional expectation $\E{X|\mu_i}$ 
is computed by the sample mean of the data points weighted by their association 
probabilities $p(\mu_i|x)$ \cite{rose1998deterministic}. 
To define an online training rule 
for the deterministic annealing framework,
a stochastic approximation algorithm can be formulated 
\cite{mavridis2022online}
to recursively estimate $\E{X|\mu_i}$ directly.
%
The following theorem follows directly from \cite{mavridis2022annealing} and 
provides an online learning rule that solves the
optimization problem of \eqref{eq:minFstar}.
%
%
\begin{theorem}
Let $\cbra{x_n}$ be a sequence of independent realizations of $X$.
Then $\mu_i(n)$, defined by the online training rule 
\begin{equation}
\begin{cases}
\rho_i(n+1) &= \rho_i(n) + \alpha(n) \sbra{ \hat p(\mu_i|x_n) - \rho_i(n)} \\
\sigma_i(n+1) &= \sigma_i(n) + \alpha(n) \sbra{ x_n \hat p(\mu_i|x_n) - \sigma_i(n)}
\end{cases}
\label{eq:oda_learning1}
\end{equation}
where $\sum_n \alpha(n) = \infty$, $\sum_n \alpha^2(n) < \infty$,
and the quantities $\hat p(\mu_i|x_n)$ and $\mu_i(n)$ 
are recursively updated 
as follows:
\begin{equation}
\begin{aligned}
\mu_i(n) = \frac{\sigma_i(n)}{\rho_i(n)},\quad
\hat p(\mu_i|x_n) = \frac{\rho_i(n) e^{-\frac{1-\lambda}{\lambda}d(x_n,\mu_i(n))}}
			{\sum_i \rho_i(n) e^{-\frac{1-\lambda}{\lambda}d(x_n,\mu_i(n))}} 
\end{aligned}
\label{eq:oda_learning2}
\end{equation}
converges almost surely to a locally asymptotically stable solution of the optimization 
\eqref{eq:minFstar}, as $n\rightarrow\infty$.
\label{thm:ODA}
\end{theorem}

\begin{proof}
    See Appendix \ref{App:online}.
\end{proof}

The learning rule
(\ref{eq:oda_learning1}), (\ref{eq:oda_learning2}) 
is a stochastic approximation algorithm \cite{borkar2009stochastic}.
In the limit $\lambda\rightarrow 0$, it results in a consistent density estimator
according to the following theorem:

\begin{theorem}
In the limit $\lambda\rightarrow 0$, and as 
the number of observed samples $\cbra{x_n}$ goes to infinity,
i.e., $n\rightarrow\infty$,
the learning algorithm based on
(\ref{eq:oda_learning1}), (\ref{eq:oda_learning2}),
results in a codebook $\mu$ that 
constructs a consistent density estimator with 
$\hat p(x) = 
\frac{\sum_i \mathds{1}_{\sbra{x\in S_i}}}{n Vol(S_i)}$, 
where 
$S_i = \cbra{x \in S: i = \argmin\limits_j ~ d(x,\mu_j)}$.
\label{thm:consistency}
\end{theorem}
\begin{proof}
See Appendix \ref{App:consistency}. 
\end{proof}

This means that as $\lambda\rightarrow 0$, the representation
of the random variable $X\in S$ by the codevectors $\mu$ 
becomes all the more accurate in $S$, according to the 
underlying probability density $p(x)$.
Regarding clustering, the nearest-neighbor rule can be used to 
partition the space $S$ in Voronoi cells
$S_i = \cbra{x \in S: i = \argmin_j ~ d_\phi(x,\mu_j)}$.

\begin{remark}
    Notice that we can express the dynamics of the codevector parameters $\mu_i(n)$
    directly as: 
    \begin{equation}
        \begin{aligned}
            \mu_i(n+1) &= 
            \frac{\alpha(n)}{\rho_i(n)}\bigg[\frac{\sigma_i(n+1)}{\rho_i(n+1)}
            (\rho_i(n)-\hat p(\mu_i|x_n))
            \\&\quad\quad\quad\quad\quad
            + (x_n \hat p(\mu_i|x_n)-\sigma_i(n))\bigg]
        \end{aligned}
    \end{equation}
    where the recursive updates take place for every codevector $\mu_i$
    sequentially. This is a discrete-time dynamical system 
    that presents bifurcation phenomena with respect to the parameter $\lambda$, 
    i.e., the number of equilibria of this system changes with respect to the
    value $\lambda$ which is hidden inside the term $\hat p(\mu_i|x_n)$ in 
    \eqref{eq:oda_learning2}.
    According to this phenomenon, the number of distinct values of $\mu_i$ is finite, 
    and the updates need only be taken with respect to these values that we call 
    ``effective codevectors''.
    This is discussed in Section \ref{sSec:bifurcation}.
\end{remark}

\subsection{Bifurcation Phenomena}
\label{sSec:bifurcation}

So far, we have assumed a countably infinite set of codevectors.
In this section we will show that the unique values of the set $\cbra{\mu_i}$ 
that solves \eqref{eq:F},
form a finite set $K(\lambda)$ of values that we will refer to as ``effective codevectors''
throughout this paper.
In other words, both the number and the locations of the codevectors depend on the value of
$\lambda$ (resp. the value of the temperature parameter $T$).
These effective codevectors are the only values that an algorithmic implementation will 
need to store in memory and update.

First, notice that when $\lambda\rightarrow 1$ 
(resp. $T\rightarrow\infty$) equation \eqref{eq:gibbs} yields
uniform association probabilities $p(\mu_i|x)=p(\mu_j|x),\ \forall i,j, \forall x$.
As a result of \eqref{eq:M}, all codevectors are located at the same point:
\begin{align*}
\mu_i = \E{X},\ \forall i
\end{align*}
which means that there is one unique effective codevector given by $\E{X}$.

As $\lambda$ is lowered below a critical value, a bifurcation phenomenon occurs, 
when the number of effective codevectors increases.
Mathematically, this occurs when the existing solution $\mu^*$ given by (\ref{eq:mu_star}) 
is no longer the minimum of the free energy $F^*$,
as $\lambda$ (resp. the temperature $T$) crosses a critical value.
Following principles from variational calculus, 
we can rewrite the necessary condition for optimality (\ref{eq:M}) as
\begin{equation}
    \frac{d}{d\epsilon} F^*(\mu+\epsilon \psi)|_{\epsilon=0} = 0
\end{equation}
with the second order condition being 
\begin{equation}
    \frac{d^2}{d\epsilon^2} F^*(\cbra{\mu+\epsilon \psi})|_{\epsilon=0} \geq 0
    \label{eq:soc}
\end{equation}
for all choices of finite perturbations $\cbra{\psi}$.
Here we denote by $\cbra{y := \mu + \epsilon \psi}$ a perturbed codebook, 
where $\psi$ are perturbation vectors applied to the codevectors $\mu$, and 
$\epsilon\geq 0$ is used to scale the magnitude of the perturbation. 
Bifurcation occurs when equality is achieved in \eqref{eq:soc} 
and hence the minimum is no longer stable\footnote{For simplicity we ignore higher order derivatives, which should be checked for mathematical completeness, but which are of minimal practical importance. The result is a necessary condition for bifurcation.}.
These conditions are described in the following theorem.
A sketch of the proof can be found in \cite{mavridis2022annealing}, and a complete
version is given in Appendix \ref{App:bifurcation}.
\begin{theorem}
Bifurcation occurs under the following condition 
\begin{equation}
    \exists y_n \text{  s.t.  } p(y_n)>0 \text{  and  } \det\sbra{ I -  \frac{1-\lambda}{\lambda} \pder{^2 \phi(y_n)}{y_n^2} C_{X|y_n}} = 0,
\end{equation}
where $C_{X|y_n} := \E{(X-y_n) (X-y_n)^\T|y_n}$.
\label{thm:bifurcation}
\end{theorem}
\begin{proof}
See Appendix \ref{App:bifurcation}.
\end{proof}

In other words, there exist critical values for $\lambda$ that depend on the data space itself and 
the choice of the Bregman divergence (through the function $\phi$), such that bifurcation occurs when 
\begin{equation}
    \frac{\lambda}{1-\lambda} = \pder{^2 \phi(y_n)}{y_n^2} \bar \nu
\end{equation}
where $\bar \nu$ is the largest eigenvalue of $C_{X|y_n}$.
That is to say that an algorithmic implementation needs only
as many codevectors as the number of effective codevectors, which
depends only on changes of the temperature parameter below certain thresholds that
depend on the dataset at hand and the dissimilarity measure used.
As shown in Alg. \ref{alg:ODA}, 
we can detect the bifurcation points 
by introducing perturbing pairs of codevectors at each 
temperature level $\lambda$ (resp. $T$).
In this way, the codevectors $\mu$ are doubled by inserting a perturbation of each $\mu_i$ in 
the set of effective codevectors.
The newly inserted codevectors will merge with their pair if 
a critical temperature has not been reached and separate otherwise.
%

\subsection{Connection to Vector Quantization. Compression Rate and Error}
\label{sSec:VQconnection}

It is apparent that problem \eqref{eq:F} is an entropy-constrained generalization 
of a soft-clustering method that, in the limit $\lambda\rightarrow 0$, converges to 
a standard vector quantization (hard-clustering) problem.
In fact, one can easily verify that
\begin{equation}
\begin{aligned}
&F_\lambda(\mu) +\lambda H(X) = (1-\lambda) D(\mu) - \lambda H(\mu) +\lambda H(X)
\\&
= \int p(x) \sum_i p(\mu_i|x) \sbra{ (1-\lambda) d(x,\mu_i) - \lambda \log p(\mu_i|x)}~dx, 
\end{aligned}
\end{equation}
and, since the entropy term $H(X)$ does not depend on the optimization parameters
$\mu$, the clustering approach in \eqref{eq:F} is equivalent to soft-clustering
with respect to a modified dissimilarity measure given by:
\begin{equation}
    d_\lambda(x,\mu_i) = (1-\lambda) d(x,\mu_i) - \lambda \log p(\mu_i|x)
\end{equation}
subject to the constraint $\sum_i p(\mu_i|x) = 1$.

Therefore, the proposed method is a lossy compression method with hierarchically 
decreasing loss as $\lambda\rightarrow 0$ and the number of effective codevectors 
goes to infinity, i.e., $K\rightarrow\infty$.
An explicit expression of the error rate $F_\lambda(\mu^*)$, 
for each temperature level $\lambda$, as a function of 
$F_0(\mu^*)=D(\mu^*)$, i.e., the error rate of a vector quantization algorithm,
is hard to obtain as it highly depends 
on the underlying distribution of the data space at hand through the entropy term.
However, an intuitive interpretation of the hierarchy of solutions that is constructed by 
solving \eqref{eq:F} for decreasing values of $\lambda$ can be seen from the form of the
conditional probabilities in \eqref{eq:gibbs}. That is, during the implementation of the
algorithm, at every level $\lambda$, a (soft-)Voronoi partition of the data space is computed with respect to a scaled dissimilarity measure: 
\begin{equation}
    \bar d_\lambda(x,\mu_i) = \frac {1-\lambda}{\lambda} d(x,\mu_i) = \frac 1 T  d(x,\mu_i)
\end{equation}
Thus, the algorithm perceives a scaled version of the data space at each level $\lambda$,
by focusing only to large dissimilarities within the data space 
when the value of $\lambda$ is high, and 
progressively zooming in to perceive more subtle dissimilarities as the value of 
$\lambda$ decreases.
Therefore, the error rate $F_\lambda(\mu^*)$ can be roughly expressed as
proportional to $D(\mu^*)$ and the term $\frac {1-\lambda}{\lambda}$ (inversely proportional to the temperature level $T$), i.e., 
$F_\lambda(\mu^*) \propto \frac {1-\lambda}{\lambda} D(\mu^*)$.

Note that this is the worst case scenario, when the introduction of the entropy term induces information loss across all regions of the data space. 
In many cases, there are regions of the data space where higher compression rate does not
introduce information loss, or the information loss is significantly lower than others.
In this sense, one can view \eqref{eq:F} as a risk-sensitive version of soft-clustering,
where an optimistic, or risk-seeking, approach is adopted. Risk-seeking in this setting 
translates to searching for less complex representations (with lower number of effective codevectors that induce higher entropy) in the hope that 
more complex representations are not necessarily needed.
The parameter $\lambda$ then becomes a weight of risk-sensitivity.
More details regarding this interpretation can be found in \cite{mavridis2022risk}.
In view of the above, the error rate $F_\lambda(\mu^*)$ can be roughly expressed as
\begin{equation}
    F_\lambda(\mu^*) \leq \frac {1-\lambda}{\lambda} D(\mu^*),\ \lambda\in[0,1).
\end{equation}
%

\subsection{Algorithmic Implementation and Complexity}
\label{sSec:Algorithm}

The progressive partitioning algorithm 
is shown in Algorithm \ref{alg:ODA}.
%
The temperature parameter $\lambda_t$ is reduced using the geometric series 
$\lambda_{t+1}=\gamma \lambda_t$, for $\gamma<1$.
Regarding the stochastic approximation stepsizes, 
simple time-based learning rates of 
the form $\alpha_n = \nicefrac{1}{a+ bn}$, $a,b>0$,
have experimentally shown to be sufficient 
for fast convergence. 
Convergence is checked with the condition
$\frac{1-\lambda}{\lambda} d_\phi(\mu_i^n,\mu_i^{n-1})<\epsilon_c$
for a given threshold $\epsilon_c$.
This condition becomes harder as the value of $\lambda$ decreases.
%
The stopping criteria $T_{stop}$ can include
a maximum number of codevectors $K_{max}$ allowed,
a minimum temperature $\lambda_{min}$ to be reached, 
a minimum distortion error $e_{target}$ to be reached,
a maximum number of iterations $i_{max}$, and so on.

Bifurcation, at $\lambda_t$, is detected by maintaining a pair 
$\cbra{\mu_j+\delta, \mu_j-\delta}$ of perturbed 
codevectors
for each effective codevector $\mu_j$ generated by the algorithm 
at $\lambda_{t-1}$,
i.e. for $j=1\ldots,K_{i-1}$.
Using arguments from variational calculus (see Section \ref{sSec:bifurcation}),
it is easy to see that, upon convegence, 
the perturbed codevectors will merge if a critical 
temperature has not been reached, and will get separated otherwise. 
Therefore, the cardinality of the model is at most doubled at 
every temperature level.
These are the effective codevectors discussed in Section \ref{sSec:bifurcation}.
%
%
Merging is detected by the condition 
$\frac{1-\lambda}{\lambda} d_\phi(\mu_j,\mu_i)<\epsilon_n$,
where $\epsilon_n$ is a design parameter
that acts as a regularization term for the model 
that controls the number of effective codevectors.
An additional regularization mechanism
is the detection of idle codevectors, which 
is checked by the condition $\rho_i(n)<\epsilon_r$, 
where $\rho_i(n)$ can be seen as an
approximation of the probability $p(\mu_i)$.

The complexity of Alg. \ref{alg:ODA} for a fixed temperature coefficient $\lambda_t$ is 
$O(N_{c_t} (2K_t)^2 d)$,
where $N_{c_t}$ is the number of stochastic approximation iterations needed for convergence 
which corresponds to the number of data samples observed, 
$K_t$ is the number of codevectors of the model at temperature $\lambda_t$, and 
$d$ is the dimension of the input vectors, i.e., $x\in\mathbb{R}^d$.
Therefore, assuming 
a schedule $\cbra{ \lambda_1=\lambda_{max}, \lambda_2, \ldots, \lambda_{N_\lambda}=\lambda_{min}}$,
the time complexity for the training of Algorithm \ref{alg:ODA} becomes:
\begin{align*}
O(N_{c} (2\bar K)^2 d)
\end{align*}
where $N_c=\max_i \cbra{N_{c_t}}$ is an upper bound on the number of data samples observed
until convergence at each temperature level, and
$\bar K = \sum_{i=1}^{N_\lambda} K_t$, with 
\begin{align*}
N_\lambda \leq \bar K \leq \min\cbra{ \sum_{n=0}^{N_\lambda-1} 2^n, \sum_{n=0}^{\log_2 K_{max}} 2^n}
< N_\lambda K_{max}
\end{align*}
where the actual value of $\bar K$
depends on the bifurcations occurred as a result of reaching critical temperatures
and the effect of the regularization mechanisms described above.
Note that typically $N_c \ll N$ as a result of the stochastic approximation algorithm,
and $\bar K \ll N_\lambda K_{max}$ as a result of the progressive nature of the algorithm.
Prediction scales linearly with $O(K_{N_\lambda} d)$, with $K_{N_\lambda}\leq K_{max}$.



\begin{algorithm}[h]
\caption{Progressive Partitioning.} 
\label{alg:ODA}
\begin{algorithmic}
\STATE Select a Bregman divergence $d_\phi$
\STATE Set stopping criteria $T_{stop}$ (e.g., $K_{max}$, $\lambda_{min}$)
\STATE Set convergence parameters: $\gamma$, 
	$\epsilon_c$, $\epsilon_n$, $\epsilon_r$, $\delta$ 
\STATE Set stepsizes: $\cbra{\alpha_n}$ 
\STATE Initialize:	$K = 1$, $\lambda = 1$,\\ \phantom{asdfasdfa}  $\cbra{\mu_0}$,  
        $p(\mu_0) = 1$, $\sigma(\mu_0) = \mu_0 p(\mu_0)$
\REPEAT
\STATE Perturb  codebook:
	$\cbra{\mu_i} \gets  
		\cbra{\mu_i+\delta} \bigcup \cbra{\mu_i-\delta}$ 
\STATE Update $K\gets 2K$,
$\cbra{p(\mu_i)}$, $\cbra{\sigma(\mu_i)\gets\mu_i p(\mu_i)}$ 
\STATE $n \gets 0$
\REPEAT 
%
\STATE Observe data point $x$ 
\FOR{$i = 1,\ldots, K$} 
\STATE Update: 
\vskip -0.3in
	\begin{align*}
	p(\mu_i|x) &\gets \frac{p(\mu_i) e^{-\frac{1-\lambda}{\lambda}d_\phi(x,\mu_i)}}
			{\sum_i p(\mu_i) e^{-\frac{1-\lambda}{\lambda}d_\phi(x,\mu_i)}} \\
	p(\mu_i) &\gets p(\mu_i) + \alpha_n \sbra{ p(\mu_i|x) - p(\mu_i)} \\
	\sigma(\mu_i) &\gets \sigma(\mu_i) + 
		\alpha_n \sbra{ x p(\mu_i|x) - \sigma(\mu_i)} \\
	\mu_i &\gets \frac{\sigma(\mu_i)}{p(\mu_i)}	
	\end{align*}
\vspace{-1.5em}
\STATE $n\gets n+1$
\ENDFOR
\UNTIL Convergence: $\frac{1-\lambda}{\lambda} d_\phi(\mu_i^n,\mu_i^{n-1})<\epsilon_c$, $\forall i $
\STATE Keep effective codevectors: \\
    \phantom{asdf} discard $\mu_i$ if $\frac{1-\lambda}{\lambda} d_\phi(\mu_j,\mu_i)<\epsilon_n$, 
	$\forall i,j,i\neq j$
\STATE Remove idle codevectors: \\ 
    \phantom{asdf} discard $\mu_i$ if $p(\mu_i)<\epsilon_r$, $\forall i$
\STATE Update $K$, $\cbra{p(\mu_i)}$, $\cbra{\sigma(\mu_i)}$
\STATE Lower temperature: $\lambda \gets \gamma \lambda$ 
\UNTIL $T_{stop}$
\end{algorithmic}
\end{algorithm}

\section{Learning with Local Models: Combined Partitioning and Function Approximation}
\label{Sec:ODA-Regression}

In this section, we investigate the problem of combined partitioning and function approximation, which 
results in a learning approach where multiple local models are trained, taking advantage of the differences 
in the underlying probability distribution of the data space.
As a consequence, this approach can circumvent the use of overly complex learning models, 
reduce time, memory, and computational complexity, and
give insights to certain properties of the data space \cite{ruping2005learning}.

In the general case, a function $f:S\rightarrow \mathcal{F}$ is to be approximated given a set of observations $\cbra{(x_n,f(x_n))}$ 
where $\cbra{x_n}$ are independent realizations of a random variable $X\in S$, similar to Section \ref{Sec:ODA}.
One then seeks to find a partition $\cbra{S_i}$ and a set of parameters $\cbra{\theta_i}\in \Theta$ for some predefined
learning models $\cbra{\hat f_i(x,\theta_i)\in\mathcal{F}}$ such that:
\begin{equation}
    \min_{\cbra{S_i,\theta_i}} \ \E{\sum_i \mathds{1}_{\sbra{X\in S_i}} d\pbra{f(X),\hat f_{i}(X,\theta_{i})}}
    \label{eq:combined}
\end{equation}
where $d:\mathcal{F}\times \mathcal{F}\rightarrow [0,\infty)$ is a well-defined convex metric with respect to the second argument.

To find a tractable solution to this problem, we decompose the two tasks of progressive partitioning and function approximation.
As described in Section \ref{Sec:ODA}, a partition $\cbra{S_i}_{i=1}^{K(\lambda)}$ of the space $S$ can be approximated 
online using a stochastic approximation algorithm that solves \eqref{eq:F}
%
%
and yields the locations of a finite number of $\mu_\lambda:=\cbra{\mu_i}_{i=1}^{K(\lambda)}$ codevectors, that
define the regions $S_i = \cbra{x \in S: i = \argmin_j ~ d_\phi(x,\mu_j)}$, $i=1,\ldots,K(\lambda)$.
Given the partition $\cbra{S_i}$, we are now in place to solve the following problem:
\begin{equation}
    \min_{\cbra{\theta_i}} \ \E{\sum_i \mathds{1}_{\sbra{X\in S_i}} d\pbra{f(X),\hat f_{i}(X,\theta_{i})}}
    \label{eq:locallearning}
\end{equation}

\begin{remark}
    Solving \eqref{eq:locallearning} decouples the two tasks of progressive partitioning and local function approximation and 
    yields a sub-optimal solution to the original combined problem in \eqref{eq:combined}.
    That being said, the use of Alg. \ref{alg:ODA} is a heuristic method that offers
    (i) the crucial properties of progressive partitioning, and 
    (ii) a compressed representation of the data space $S$
    such that each $S_i$ represents a region of $S$ where its underlying probability distribution 
    presents low variability (see Section \ref{Sec:ODA}).
\end{remark}

In the remaining section, we will study learning approaches to computationally solve 
\eqref{eq:locallearning} in the general case of a differentiable (with respect to $\theta_i$)
learning model $\hat f_i(x,\theta_i)$, and in the specific case of using locally constant models, i.e., 
when $\hat f(x,\theta_i)=\theta_i\in \mathcal{F}$.

\subsection{Learning with Local Models}
\label{sSec:General-regression}

In this section, we assume a model 
$\hat f_i(x,\theta_i)$ $\in \mathcal{F}$ that is differentiable with respect to a parameter vector $\theta_i\in\Theta$,
where $\Theta$ is a finite-dimensional vector space.
Given a finite partition set of parameters $\cbra{S_i}_{i=1}^{K(\lambda)}$, for $K(\lambda)<\infty$, \eqref{eq:locallearning} is decomposed to
\begin{equation}
    \min_{\theta_i} \ \E{\mathds{1}_{\sbra{X\in S_i}} d\pbra{f(X),\hat f_{i}(X,\theta_{i})}}, \ i=1,\ldots,K(\lambda).
    \label{eq:ll}
\end{equation}
where $d:\mathcal{F}\times \mathcal{F}\rightarrow [0,\infty)$ is assumed a metric that is differentiable and convex with respect to the second argument.

This is a stochastic optimization problem that can be solved using stochastic approximation updates.
In particular, since we have assumed that $\hat f_i(x,\theta_i)$ is differentiable with respect to $\theta_i$,
we can use stochastic gradient descent:
\begin{equation}
\begin{aligned}
&\theta_i(n+1) = \theta_i(n) - \beta(n) \nabla_\theta d(f(x_n),\hat f_i(x_n,\theta_i(n)))
\\&\quad
= \theta_i(n) - \beta(n) \{ \nabla_\theta \E{\mathds{1}_{\sbra{X\in S_i}} d(f(x_n),\hat f_i(x_n,\theta_i(n)))} 
\\ &\quad\quad 
+\big( \nabla_\theta d(f(x_n),\hat f_i(x_n,\theta_i(n))) 
\\ &\quad\quad 
- \nabla_\theta \E{\mathds{1}_{\sbra{X\in S_i}} d(f(x_n),\hat f_i(x_n,\theta_i(n)))} \big) \}     , \ x_n\in S_i 
\end{aligned}
\label{eq:theta-update}
\end{equation}
Since we can control the observations for each model $f_i$ to belong to $S_i$, it is easy to see that 
$M_{n+1}:= \nabla_\theta d(f(x_n),\hat f_i(x_n,\theta_i(n))) - \nabla_\theta \E{\mathds{1}_{\sbra{X\in S_i}} d(f(x_n),\hat f_i(x_n,\theta_i(n)))}$
is a martingale difference sequence for an unbiased estimator $\nabla_\theta d(f(x_n),\hat f_i(x,\theta_i))$, i.e., when the condition 
$\E{ \mathds{1}_{\sbra{X\in S_i}} \nabla_\theta d(f(x_n),\hat f_i(x_n,\theta_i(n)))} = \nabla_\theta \E{\mathds{1}_{\sbra{X\in S_i}} d(f(x_n),\hat f_i(x_n,\theta_i(n)))}$ holds.
Therefore, as an immediate result of 
Theorem \ref{thm:borkar} in Appendix \ref{App:online}, 
the stochastic approximation process \eqref{eq:theta-update} converges almost surely to a possibly
path-dependent invariant set of $\dot \theta_i = \nabla_\theta \E{\mathds{1}_{\sbra{X\in S_i}} \hat f_i(x,\theta_i)}$, i.e., 
an asymptotically stable local minimum of the objective function $\E{\mathds{1}_{\sbra{X\in S_i}} \hat d(f(x_n),f_i(x,\theta_i))}$.

So far, we have assumed that $\cbra{S_i}$ is fixed. 
However, we are interested in a learning approach that approximates $\cbra{S_i}$ and $\cbra{\hat f_i(x,\theta_i)}$
at the same time, and given the same observations $\cbra{(x_n,f(x_n))}$ which may be available one at a time (i.e, no dataset is stored in memory a priori).
This is possible because both learning algorithms for $\cbra{S_i}$ and $\cbra{\hat f_i(x,\theta_i)}$ independently are stochastic approximation algorithms.
According to the theory of two-timescale stochastic approximation, we can run both learning algorithms at the same time, but using different stepsize
profiles $\cbra{\alpha(n)}$ and $\cbra{\beta(n)}$, such that $\nicefrac{\alpha(n)}{\beta(n)}\rightarrow 0$.
Intuitively, we create a system of two dynamical system running in different ``speed'', 
meaning that second system, the one with stepsizes $\cbra{\beta(n)}$, is updated fast enough that the first system, 
the one with stepsizes $\cbra{\alpha(n)}$, can be seen as quasi-static with respect to the second.
The following theorem summarizes this result.

\begin{theorem}
Let $\cbra{x_n}$ be a sequence of independent realizations of $X$, and 
assume that $\mu_i(n)$ is a sequence updated using the stochastic approximation algorithm in \eqref{eq:oda_learning1}
with stepsizes $\cbra{\alpha(n)}$ satisfying $\sum_n \alpha(n) = \infty$, and $\sum_n \alpha^2(n) < \infty$.
Then, as long as $\cbra{\beta(n)}$ are designed such that $\sum_n \beta(n) = \infty$, $\sum_n \beta^2(n) < \infty$,
and $\nicefrac{\alpha(n)}{\beta(n)}\rightarrow 0$, the asynchronous updates
\begin{equation}
\theta_i(n+1) = \theta_i(n) - \beta(n) \nabla_\theta d(f(x_n),\hat f_i(x_n,\theta_i(n))), 
\label{eq:oda-regression}
\end{equation}
for $i = \argmin_j ~ d_\phi(x_n,\mu_j(n))$ converges almost surely to a locally asymptotically stable solution $\cbra{\theta_i}$ of \eqref{eq:ll}, as $n\rightarrow\infty$,
for $S_i = \{x \in S: i = \argmin_j ~ d_\phi(x,\mu_j(\infty))\}$, where
$\mu_i(\infty))$ is the asymptotically stable equilibrium of \eqref{eq:oda_learning1}.
\label{thm:ODA2timescale}
\end{theorem}

\begin{proof}
    See Appendix \ref{App:online-regression}.
\end{proof}

The algorithmic implementation is shown in Alg. \ref{alg:ODA-regression}
as an extension of Alg. \ref{alg:ODA}.

\begin{algorithm}[h]
\caption{Progressive Learning with Differentiable Models.}
\label{alg:ODA-regression}
\begin{algorithmic}
\STATE -----$//$-----
\STATE Set stepsizes: $\cbra{\alpha_n}$, \hl{$\cbra{\beta_n}$ s.t.  $\nicefrac{\alpha_n}{\beta_n}\rightarrow 0$ }
\STATE Initialize: $\cbra{\mu_0}$, \hl{$\cbra{\theta_0}$}
\REPEAT
\STATE -----$//$-----
\REPEAT 
\STATE Observe data point $x$ \hl{\& output $y$}
\FOR{$i = 1,\ldots, K$} 
\STATE Update:
\STATE -----$//$-----
\vskip -0.2in
	\begin{align*}
	p(\mu_i) &\gets p(\mu_i) + 
 \alpha_n \sbra{ p(\mu_i|x) - p(\mu_i)} \\
	\sigma(\mu_i) &\gets \sigma(\mu_i) + 
		\alpha_n \sbra{ x p(\mu_i|x) - \sigma(\mu_i)}
	\end{align*}
\vskip -0.1in
\phantom{asfda} \hl{ 
	$\theta_i \gets \theta_i - \beta_n \nabla_\theta d(f(x),\hat f_i(x,\theta_i)) 
$}
\vskip -0.5in
\STATE -----$//$-----
\vskip -0.3in
\ENDFOR
\UNTIL Convergence
\STATE -----$//$-----
\UNTIL $T_{stop}$
\end{algorithmic}
\end{algorithm}

\subsection{Case of Constant Local Models}
\label{sSec:PWC-regression}

In the special case when locally constant models are used, i.e., 
when $\hat f(x,\theta_i)=\theta_i\in \mathcal{F}$, 
two-timescale updates are not required, and a simpler solution can be tracked.
In particular, we can augment the system \eqref{eq:oda_learning1} with 
\begin{equation}
\begin{cases}
\sigma_{\theta_i}(n+1) &= \sigma_{\theta_i}(n) + \alpha(n) \sbra{ x_n \hat p(\mu_i|x_n) - \sigma_{\theta_i}(n)} \\
\theta_i (n) &= \frac{\sigma_{\theta_i}(n)}{\rho_i(n)}
\end{cases}
\label{eq:oda-regression-pwc}
\end{equation}
Following the same arguments as in the proof of Theorem \ref{thm:ODA}, it is easy to see that 
$\theta_i(n)$ converge almost surely to $\E{\mathds{1}_{\sbra{X\in S_i}} f(X)}$ as $n\rightarrow\infty$ and $\lambda\rightarrow 0$.
To see this, notice that as $\lambda\rightarrow 0$, $p^*(x,\mu_i)\rightarrow \mathds{1}_{\sbra{X\in S_i}}$
and $p^*(\mu_i)\rightarrow 1$.
As a final note, this approach is equivalent to a piece-wise constant approximation of $f(X)$.
In other words, this is a binning process where the size and location of the bins depends on the underlying probability 
distribution of $X$, and the number of bins progressively increases, resulting in a hierarchical approximation of $f(X)$.
The algorithmic implementation is shown in Alg. \ref{alg:ODA-regression-constant}
as an extension of Alg. \ref{alg:ODA}.

%
\begin{algorithm}[h]
\caption{Progressive Learning with Constant Models.} 
\label{alg:ODA-regression-constant}
\begin{algorithmic}
\STATE -----$//$-----
\STATE Initialize: $\cbra{\mu_0}$, \hl{$\cbra{f_{\mu_0}}$, $\cbra{\sigma_f(\mu_0)}$ }
\REPEAT
\STATE -----$//$-----
\REPEAT 
\STATE Observe data point $x$ \hl{\& output $y$}
\FOR{$i = 1,\ldots, K$} 
\STATE Update: 
\STATE -----$//$-----
        \\\phantom{asdfa}\hl{
	$\sigma_f(\mu_i) \gets \sigma(\mu_i) + 
		\alpha_n \sbra{ y p(\mu_i|x) - \sigma(\mu_i)}$}
\STATE  
 \phantom{asdfa}
 $\mu_i \gets \frac{\sigma(\mu_i)}{p(\mu_i)}$,  \hl{$f_{\mu_i} \gets \frac{\sigma_f(\mu_i)}{p(\mu_i)}$	}
\ENDFOR
\UNTIL Convergence
\STATE -----$//$-----
\UNTIL $T_{stop}$
\end{algorithmic}
\end{algorithm}
%

\section{The Problem of Classification}
\label{Sec:ODA-Classification}

In this section we focus on the binary classification problem. 
The results can be extended to the general case (see, e.g., \cite{devroye2013}).
For the classification problem, 
a pair of random variables 
$\cbra{X,c(X)} \in S\times \cbra{0,1}$ defined in a probability space
	$\pbra{\Omega, \mathcal{F}, \mathbb{P}}$, is observed with
	$c(X)$ representing the class of $X$ and $S\subseteq\mathbb{R}^d$.	
	The codebook is represented by
	$\mu := \cbra{\mu_i}_{i = 1}^K$, $\mu_i \in ri(S)$, and
	$c_\mu := \cbra{c_{\mu_i}}_{i = 1}^K$,
	such that $c_{\mu_i} \in \cbra{0,1}$ 
	represents the class of $\mu_i$ for all $i \in \cbra{1,\ldots,K}$.
	%
	%
	A partition-based classifier is called Bayes-optimal if it minimizes the classification error:
\begin{equation}
\begin{aligned}
    \min_{\mu,c_\mu} ~ J_B(\mu,c_\mu) &:= 
  \pi_1 \sum_{i:c_{\mu_i}=0} \mathbb{P}\cbra{X\in S_i | c(X) = 1} 
\\&\quad		
  +\pi_0 \sum_{i:c_{\mu_i}=1} \mathbb{P}\cbra{X\in S_i | c(X) = 0} 
\end{aligned}
\label{eq:lvq}
\end{equation}	
where $S_i = \cbra{x \in S: i = \argmin\limits_j ~ d(x,\mu_j)}$,
and $\pi_i := \mathbb{P}\sbra{c = i}$.

In the remaining section, we study methods to solve the classification problem based on the results of 
Sections \ref{Sec:ODA} and \ref{Sec:ODA-Regression}.

\subsection{Classification as a Regression Problem with Constant Local Models}

The classification problem \eqref{eq:lvq} can be viewed as a special case of learning with local models as in Section \ref{Sec:ODA-Regression}.
Here one seeks to find a partition $\cbra{S_i}$ and a set of parameters $\cbra{c_i\in \cbra{0,1}}$ such that:
\begin{equation}
    \min_{\cbra{S_i,c_i}} \ \E{\sum_i \mathds{1}_{\sbra{X\in S_i}} d\pbra{c(X),c_i}}
    \label{eq:combinedClass}
\end{equation}
where $d:=\mathds{1}_{\sbra{c\neq c_{\mu_i}}}$.
Notice that since $d$ is not differentiable the results of Section \ref{Sec:ODA-Regression} cannot be used directly.
However, numerous relaxation methods can be used to find a possibly sub-optimal solution. 
A widely used approach is to relax the constraints on $\cbra{c_i\in \cbra{0,1}}$ such that $\cbra{c_i\in \sbra{0,1}}$.
Then the updates \eqref{eq:oda-regression-pwc} can be directly used to estimate $\E{\mathds{1}_{\sbra{X\in S_i}} c(X)}$.
Then a projection mapping $r:\sbra{0,1}\rightarrow \cbra{0,1}$, e.g., $r(c) = \mathds{1}_{\sbra{c<0.5}}$, can be used to return a solution 
to the classification problem. 
Notice that this is equivalent to a majority-vote rule inside each region $S_i$. 
This is a common approach that, at the limit $\lambda\rightarrow 0$, when the updates \eqref{eq:oda_learning1}, \eqref{eq:oda_learning2}
result in a hard-clustering approach with infinite number of clusters,
yield a classification rule that is strongly Bayes 
risk consistent, i.e., converges to the optimal (Bayes) probability 
of error given in \eqref{eq:lvq} (see, e.g., Ch. 21 in \cite{devroye2013}).

\subsection{Classification as Class-Conditioned Density Estimation}
\label{sSec:classification}

In a different approach, we can formulate the binary classification problem to
the minimization of $F$ in 
(\ref{eq:F}) with a modified average distortion measure given by:
\begin{align*}
D = \E{d^b(X,c,\mu,c_\mu)} 
\end{align*}
where $ d^b(x,c,\mu_i,c_{\mu_i}) = \mathds{1}_{\sbra{x\in S_i}} \mathds{1}_{\sbra{c\neq c_{\mu_i}}}$.
However, because $d^b$ is not differentiable, 
using similar principles as in the case of Learning Vector Quantization (LVQ)
\cite{Kohonen1995}, 
we can instead approximate the optimal solution by 
using the distortion measure
\begin{equation}
d^l(x,c_x,\mu,c_\mu) = \begin{cases}
				d(x,\mu),~ c_x=c_\mu \\
				-d(x,\mu),~ c_x\neq c_\mu			
				\end{cases}
\label{eq:dl}
\end{equation}
Using similar arguments to Ch. 21 in \cite{devroye2013},
it can be shown that as $\lambda\rightarrow 0$, 
the solution $(\mu,c_\mu)$ to the above problem
equipped with a majority-vote classification rule is strongly Bayes 
risk consistent.

However, we find useful to also explore a generative learning approach, using
\begin{equation}
d^c(x,c_x,\mu,c_\mu) = \begin{cases}
				d(x,\mu),~ c_x=c_\mu \\
				0,~ c_x\neq c_\mu			
				\end{cases}
\label{eq:dc}
\end{equation}
It is easy to see that this particular choice for the distortion measure $d^c$ 
in (\ref{eq:dc}) transforms the learning rule in (\ref{eq:oda_learning1}) to 
\begin{equation}
\begin{cases}
\rho_i(n+1) &= \rho_i(n) + \beta(n) \sbra{ s_i \hat p(\mu_i|x_n) - \rho_i(n)} \\
\sigma_i(n+1) &= \sigma_i(n) + \beta(n) \sbra{ s_i x_n \hat p(\mu_i|x_n) - \sigma_i(n)}
\end{cases}
\label{eq:oda_learning1c}
\end{equation}
where $s_i:=\mathds{1}_{\sbra{c_{\mu_i}=c}}$.
As a result, this is equivalent to estimating strongly consistent
class-conditional density estimators:
\begin{equation}
     \hat p(x|c=j) \rightarrow \pi_j p(x|c=j) ,\ a.s.
\end{equation}
and the following theorem holds:
\begin{theorem}
In the limit $\lambda\rightarrow 0$, and as 
the number of observed samples $\cbra{x_n}$ goes to infinity,
i.e., $n\rightarrow\infty$,
the learning algorithm based on
(\ref{eq:oda_learning1c}), (\ref{eq:oda_learning2}),
results in strongly consistent
class-conditional density estimators $\hat p(x|c=j)$
that construct a Bayes risk consistent classifier with
the classification rule 
\begin{equation}
\hat c = \argmax_ j \hat \pi_j \hat p(x|c=j),\ j=1,2
\end{equation}
where 
$\hat \pi_j = \frac{\sum_n \mathds{1}_{\sbra{c_n=j}}}{n}$
\label{thm:bayes}
\end{theorem}
\begin{proof}
See Appendix \ref{App:bayes}.
\end{proof}

As a final note, an easy-to-implement nearest-neighbor rule classification rule:
\begin{equation}
\hat c(x) = c_{\mu_{h^*}}
\end{equation}
where $h^* = \argmax\limits_{\tau = 
		1,\ldots,K} ~ p(\mu_\tau|x),~ h \in \cbra{1, \ldots, K}$,
yields a classification error $\hat J_B^*$ with tight upper bound
with respect to the Bayes-optimal $J_B^* $, i.e.,  
$ J_B^* \leq \hat J_B^* \leq 2 J_B^*$ (see, e.g., \cite{devroye2013}).
The algorithmic implementation is shown in Alg. \ref{alg:ODA-class}
as an extension of Alg. \ref{alg:ODA}.

\begin{algorithm}[H]
\caption{Progressive Classification via Class-Conditional Density Estimation.} 
\label{alg:ODA-class}
\begin{algorithmic}
\STATE -----$//$-----
\STATE Initialize: $\cbra{\mu_0}$  \hl{\& $\cbra{c_{\mu_0}}$}, \\ 
\REPEAT
\STATE -----$//$-----
\REPEAT 
%
\STATE Observe data point $x$ \hl{\& class label $c$}
\IF {\hl{$\nexists \mu_i$ s.t. $c_{\mu_i}=c$}}
\STATE \hl{Insert: $\cbra{\mu_i} \leftarrow \cbra{\mu_i} \bigcup \cbra{x}$}
\STATE \phantom{Insert:} \hl{$\cbra{c_{\mu_i}} \leftarrow \cbra{c_{\mu_i}} \bigcup \cbra{c}$}
\ENDIF
\FOR{$i = 1,\ldots, K$} 
\STATE \hl{Compute membership $s_i = \mathds{1}_{\sbra{c_{\mu_i}=c}}$} 
\STATE Update: 
\STATE -----$//$-----
\\\phantom{asdfd}
$p(\mu_i) \gets p(\mu_i) + \alpha_n \big[ $\hl{$s_i$}$ p(\mu_i|x) - p(\mu_i)\big]$
\\\phantom{asdfd}
$\sigma(\mu_i) \gets \sigma(\mu_i) + 
\alpha_n \big[ $\hl{$s_i$}$ x p(\mu_i|x) - \sigma(\mu_i)\big]$
\STATE -----$//$-----
\ENDFOR
\UNTIL Convergence
\STATE -----$//$-----
\UNTIL $T_{stop}$
\end{algorithmic}
\end{algorithm}
%

\section{Hierarchical Learning in Multiple Resolutions}
\label{Sec:MR-ODA}

%
%

In this section, we extend the progressive partitioning algorithm (Alg. \ref{alg:ODA}) of Section \ref{Sec:ODA},
by imposing a tree structure in the construction of the regions $\cbra{S_i}$.
The results of Sections \ref{Sec:ODA-Regression}, \ref{Sec:ODA-Classification}, are extended naturally through their 
immediate dependence on the partition $\cbra{S_i}$. 
The key idea of the progressive construction of the tree structure is as follows.
Given a value for the temperature coefficient $\lambda_t$, Algorithm \ref{alg:ODA}, as presented in Section \ref{Sec:ODA},
yields a sequence of partitions $\cbra{S_i}_{i=1}^{K(\lambda^t)}$.
%
%
If at $\lambda^t$, a user-defined splitting criterion is met, the partition $\cbra{S_i}_{i=1}^{K(\lambda^t)}$ is fixed, and
Algorithm \ref{alg:ODA} is applied independently to each region $S_i$ to create $\cbra{\cbra{S_{ij}}_{j=1}^{K_i(\lambda^{t_i})}}_{i=1}^{K(\lambda^t)}$, 
such that $\cbra{S_{ij}}_{j=1}^{K_i(\lambda^{t_i})}$ form a partition of $S_i$.
This is depicted in Fig. \ref{sfig:tree-regression}.
For each parent set $S_i$, the number of children sub-sets $K_i(\lambda^{t_i})$, may be different, depending on the properties of $S_i$. 
The same holds for the stopping values $\lambda^{t_i}$.
The splitting criterion 
can involve terms such as 
a minimum value of $\lambda_{min}$ reached, 
a maximum number of $K_{max}$ codevectors reached, or
a minimum percentage of improvement in accuracy or distortion reduction 
for every temperature step is reached.
This structural constraint reduces the time complexity of the 
algorithm from $O(K^2)$ to $O(k^2 + \log_{k}K)$, where $K$ here represents the total number of sets $\cbra{S_i}_{i=1}^K$, and 
$k$ represents the number of children sub-sets for each parent set (assumed equal for every parent set)
\cite{gray_VQ_1990}.
In addition, as we will show, this tree structure offers 
an inherent regularization mechanism in classification applications (Section \ref{Sec:ODA-Classification}).
Finally, since the resulting structure is a non-binary tree-structure, 
we are able to control the number of layers of the tree-structured partition of the data space, without sacrificing the performance of the learning algorithm, i.e.,
a finite tree depth is sufficient for convergence \cite{riskin1991greedy,nobel1996termination}.
Therefore, we can match the number of tree layers to the number of resolutions in a multi-resolution data representation.
This will allow for training each layer of the tree with progressively finer resolution of the data representation, 
which defines a hierarchical and progressive learning approach that further reduces the complexity of the algorithm, while 
inheriting potential benefits from the feature extraction process of the multi-resolution analysis.
As we will show, in the case when group-convolutional wavelet transform is used to create the multi-resolution data representation,
this architecture shares similar properties to a deep neural network architecture
\cite{mallat2016understanding}.

\subsection{Tree-Structured Progressive Partitioning}
\label{sSec:tsoda}

A tree-structured partition $\Sigma_\Delta:=\cbra{S_{\nu_i}}$ 
is defined by a set of regions $S_{\nu_i}\in S$, each represented by a tree node $\nu_i$,
arranged in a tree structure $\Delta$ with a single root node $\nu_0$ such that $S_{\nu_0}= S$.
%
%
The tree structure $\Delta$ is a special case of a connected, acyclic directed 
graph, where each node has a single parent node (except for the root node)
and an arbitrary number of children nodes, that is, $\Delta$ is not restricted to be a binary tree. 
The set $C(\nu_i)$ represents the nodes $\cbra{\nu_j}$ that are children of $\nu_i$, 
while the set $P(\nu_j)$ represents the node $\nu_i$ for which $\nu_j \in C(\nu_i)$.
The level $l\geq 0$ of a node $\nu_h \in \Delta$ 
is the length of the path $\cbra{\nu_0, \ldots, \nu_i, \nu_j, \ldots, \nu_h}$
leading from the root node $\nu_0$ to $\nu_h$ such that $\nu_j \in C(\nu_i)$.
%
%
The terminal nodes $\tilde \nu := \cbra{\nu_i: C(\nu_i)=\emptyset}$ are called leaves, 
and the union of their associated sets will be denoted
$\tilde S:=\cbra{\tilde S_{j}}$, 
where $|\tilde S|=\tilde K$ is the number of leaf sets that create a partition of $S$, 
and $\tilde l:=\max \cbra{l:\nu_i^{(l)}\in\tilde \nu} <\infty$ 
will denote the maximum depth of the tree.

$\Sigma_\Delta$
defines a hierarchical 
partitioning scheme for the domain $S$, such that
for every node $\nu_i \in \Delta$ associated with the region $S_{\nu_i}$, 
its children nodes $\cbra{\nu_j \in C(\nu_i)}$ are associated with
the regions $\cbra{S_{\nu_j}}$ that form a partition of $S_{\nu_i}$.
We will use the unique paths from the root node as identification label for each node, i.e., $\nu_j = 0\ldots ij$ such that $\nu_0 = 0$, 
$C(0) = \cbra{0i}$, $C(0i) = \cbra{0ij}$, and so on.
As such, Algorithm \ref{alg:ODA} can be used recursively to construct a tree-structured partition $\Sigma_\Delta$ as follows:
Start with node $\nu_0=0$ as the only leaf node.
Using observations $\cbra{x_n}$ (realizations of $X\in S$), apply Algorithm \ref{alg:ODA} until  
a partition $\cbra{S_{0j}}$ of $S_0 = S$ is constructed. 
Then starting with $w=0$ and for every observation $x_n$, 
iterate the process 
\begin{equation}
    \text{repeat } w \gets w^\prime \in C(w) \text{ such that } x_n\in S_{w^\prime} \text{, until } C(w)=\emptyset,
    \label{eq:tree-iteration}
\end{equation}
%
and apply one stochastic approximation update of Algorithm \ref{alg:ODA} in $S_{w}$. 
This asynchronous process can continue until the convergence of 
all applications of Alg. \ref{alg:ODA}, when a finite-depth tree-structured partition $\Sigma_\Delta$ is constructed 
such that for every node $w\in \Delta$ with children nodes $\cbra{wj} \in C(w)$, 
the regions $\cbra{S_{wj}}$ form a partition of $S_{w}$.
This process is illustrated in Algorithm \ref{alg:ts-oda} and its asymptotic behavior is given by the following theorem:

\begin{theorem}
Let $\Sigma_\Delta$
be a finite-depth tree-structured 
partitioning scheme created by Alg. \ref{alg:ts-oda} using realizations $\cbra{x_n}$ of a random variable $X\in S$.
If the leaf nodes
are updated at the limit $\lambda\rightarrow 0$, and $n\rightarrow\infty$,
then $\Sigma_\Delta$ yields a consistent density estimator 
of $X$, with 
$\hat p(x) = 
\frac{\sum_i \mathds{1}_{\sbra{x\in \tilde S_i}}}{n Vol(\tilde S_i)}$, 
where $\tilde S_i$ is a leaf node given by the iterative process \eqref{eq:tree-iteration}.
\label{thm:tree-consistency}
\end{theorem}
\begin{proof}
It follows directly by the application of Theorem \ref{thm:consistency} to each region $\tilde S_{j}$,
where $\cbra{\tilde S_j}$ are the leaf nodes of $\Sigma_\Delta$ that form a partition of $S$. 
\end{proof}

\begin{remark}
    We have shown the tree-structured extension of Alg. \ref{alg:ODA},
    as well as its asymptotic behavior for finite tree depth.
    It is straightforward to show that similar results hold for the tree-structured extension of Algorithms \ref{alg:ODA-regression},
    \ref{alg:ODA-regression-constant}, and \ref{alg:ODA-class}, 
    regarding the regression and classification problems discussed 
    in Sections \ref{Sec:ODA-Regression}, and \ref{Sec:ODA-Classification}.
\end{remark}

Notice that a finite tree depth is sufficient for convergence to a consistent density estimator. This follows from the fact that the children of each tree node $\nu_i$, representing a region 
$S_{\nu_i}\in S$, are the output of the progressive construction of a partition of $S_{\nu_i}$, based on Alg. \ref{alg:ODA}.
This result will be used in Section \ref{sSec:mroda} to build 
a multi-resolution extension of Algorithm \ref{alg:ts-oda}.

The time complexity of the tree-structured algorithm is significantly reduced. 
Let $K_{max}$ be the total number of codevectors allowed.
Then Alg. \ref{alg:ODA} has a worst-case complexity  
$O(N_{c} (2\bar K)^2 d)$, for training, where $\bar K = \sum_{n=0}^{\log_2 K_{max}} 2^n$, while testing requires
$O(K_{max}d)$ (see Section \ref{sSec:Algorithm} for details on 
the parameters).
In a tree-structured partition $\Sigma_\Delta$
of depth $\tilde l$,
assuming that the number of
children $k=|C(\nu_i)|$
of each node $\nu_i$ is $k$ is the same, 
and that each region is represented by roughly the same number of 
observations $N_{c}^{\tilde l}$, 
we get $k=(K_{max})^{\nicefrac{1}{\tilde l}}$, and 
$N_{c}^{\tilde l} = \nicefrac{N_{c}}{k}$.
Then training requires in the worst-case:
\begin{align*}
O \pbra{ \frac{k^{\tilde l}-1}{k(k-1)} N_{c} (2\bar k)^2 d)}
\end{align*}
where $\bar k = \sum_{n=0}^{\log_2 k} 2^n= \sum_{n=0}^{\nicefrac{1}{\tilde l} \log_2 K_{max}}2^n$.
Prediction requires a forward pass of the tree, i.e., 
it scales with $O(k \log_k K_{max} d)$.

In addition, we note that Alg. \ref{alg:ts-oda}
updates the partition $\cbra{S_{wj}}$ of each node $w$ asynchronously.
As a result, depending on the underlying probability density of the random variable $X\in S$, some nodes  
will be visited more often than others, which will result in
some branches of the tree growing faster than others, 
inducing a variable-rate code that
frequently outperforms fixed-rate, full-search techniques with
the same average number of bits per sample.
Alternatively, when learning offline using a dataset, 
all nodes can be trained using parallel processes,
which can be utilized by multi-core computational units.

In the classification problem, an additional regularization 
mechanism can be added in the approach described in 
Section \ref{sSec:classification} (Alg. \ref{alg:ODA-class}).
%
Specifically, 
when a partition $\cbra{S_{wj}}_{j=1}^{K_w}$ of a node $w$ is fixed, the node $w$ can check the condition $c_{\mu_{wi}} = c_{\mu_{wj}}, \forall i,j$, which means that the partition $\cbra{S_{wj}}_{j=1}^{K_w}$ is using $K_w$ codevectors, all of which
correspond to the same class. 
In this case, node $w$ is assigned a single codevector, and is not further split by the algorithm.
This phenomenon is illustrated in Fig. \ref{fig:my2d-c-111} in Section \ref{Sec:Results}.

Finally, we note that the termination criteria of the iterations of 
Alg. \ref{alg:ODA} in Alg. \ref{alg:ts-oda} 
in each layer of the tree are important design parameters. 
These can include 
a maximum number of codevectors for each partition,
a minimum temperature $\lambda_{min}^l$ in each tree layer $l$, 
a maximum number of iterations, and so on.
These termination criteria characterize the splitting criteria as well, i.e., 
when to stop growing the set of effective codevectors in a partition of layer $l$, 
and continue to split the partition in the next layer $l+1$.

\begin{algorithm}[H]
\caption{Tree-Structured Progressive Partitioning}
\begin{algorithmic}
\STATE Initialize root node $\nu_0$ s.t. $S_{\nu_0} = S$
\REPEAT 
\STATE Observe data point $x\in S$
\STATE Find leaf node to update:
\STATE Set $w = \nu_0$
\WHILE {$C(w)\neq \emptyset$}
\STATE $w \gets v \in C(w)$ such that $x \in S_{v}$
\ENDWHILE 
\STATE Update partition $\cbra{S_{wj}}$ of $S_{w}$ using $x$ and Alg. \ref{alg:ODA} 
\IF {Alg. \ref{alg:ODA} in $S_{w}$ terminates}
\STATE Split node $w$: $C(w)\gets \cbra{S_{wj}}$
\ENDIF
\UNTIL Stopping criterion 
\end{algorithmic}
\label{alg:ts-oda}
\end{algorithm}

\subsection{Multi-Resolution Extension}
\label{sSec:mroda}

So far we have modeled the observations as realizations of a random variable $X\in S\subseteq \mathbb{R}^d$.
In general, $X$ can be itself a measurable signal $X(t):\mathbb{R}^n \rightarrow S$ with finite energy, i.e., 
$X(t) \in S \subseteq L^2(\mathbb{R}^d)$. 
We will denote the original space $S$ as $S^0$.
A multi-resolution representation of the signal $X(t)$ consists of 
a sequence of projections of $X(t)$ on
subspaces $\cbra{S^j}$ such that 
$S^j\subset S^{j-1}$, $\forall j\in \mathbb{N}$,
and $\cup_{j=0}^{\infty} S^j$ is dense in $S^0$
	with $\cap_{j=0}^{\infty} S^j=\cbra{0}$.
There are numerous methods to construct subspaces $\cbra{S^j}$
with these properties, from the classical wavelet transform
\cite{mallat1999wavelet} to different dictionary learning approaches \cite{chen2001atomic,lecun2016thenext}.
An approach using group-convolutional wavelet decomposition
will be presented in Section \ref{sSec:scattering}.

We denote by $X^r\in S^r$ the projection of $X=X^0$ to the subspace 
$S^r$. 
Given a multi-resolution representation of $X$ with subspaces 
$\cbra{ S^0, S^1, S^2, \ldots, S^{\tilde l} }$, 
we can extend Alg. \ref{alg:ts-oda} presented in Section \ref{sSec:tsoda} 
such that $X^{\tilde l - r} \in S^{\tilde l - r}$ is used
to train the nodes of the tree at level $r$. 
This idea matches the intuition of 
using higher-resolution representation of $X$ for deeper layers of the tree.
It was first introduced in 
\cite{baras1993efficient} and \cite{baras1994wavelet},
and constitutes a hierarchical multi-resolution learning scheme.  
The algorithmic implementation is straightforward given Alg. \ref{alg:ts-oda},
and is given in Alg. \ref{alg:mr-oda}.

Regarding the asymptotic behavior of the multi-resolution extension, 
the results of Theorem \ref{thm:tree-consistency} hold.
To see that, notice that since $S^{l}\subset S^0$, for $l>0$, 
it follows that $X^l\in S^0$ as well, 
and, as a result, Alg. \ref{alg:mr-oda} essentially creates 
a tree-structured partition $\Sigma_\Delta$ of $S^0=S$, 
with the leaf nodes trained with $X^0\in S^0=S$, such that
the following holds.

\begin{theorem}
Let $\Sigma_\Delta$
be a tree-structured 
partitioning scheme of depth $\tilde l$ 
created by Alg. \ref{alg:mr-oda} using the multi-resolution representation 
$\pbra{x_n^0=x_n, x_n^1, \ldots, x_n^{\tilde l}} \in \pbra{S^0=S, S^1, \ldots, S^{\tilde l}}$ of realizations of a random variable $X\in S$.
If the leaf nodes
are updated at the limit $\lambda\rightarrow 0$, and $n\rightarrow\infty$,
then $\Sigma_\Delta$ yields a consistent density estimator of $X\in S$, with $\hat p(x) = 
\frac{\sum_i \mathds{1}_{\sbra{x\in \tilde S_i}}}{n Vol(\tilde S_i)}$, 
where $\tilde S_i$ is a leaf node given by the iterative process \eqref{eq:tree-iteration}.
\label{thm:mr-consistency}
\end{theorem}
\begin{proof}
It follows directly by the application of Theorem \ref{thm:consistency} to each region $\tilde S_{j}$,
where $\cbra{\tilde S_j}$ are the leaf nodes of $\Sigma_\Delta$ that form a partition of $S$. 
\end{proof}

%

\begin{algorithm}[t]
\caption{Multi-Resolution Progressive Partitioning}
\begin{algorithmic}
\STATE Initialize root node $\nu_0$ s.t. $S_{\nu_0} = S^{\tilde l}$
\REPEAT 
\STATE Observe data point $x^{\tilde l} \in S^{\tilde l}$
\STATE Find leaf node to update:
\STATE Set $w = \nu_0$
\STATE Set resolution $l=\tilde l$
\WHILE {$C(w)\neq \emptyset$}
\STATE $w \gets v \in C(w)$ such that $x^{l} \in S_{v}^{l}$
\STATE $l \gets l-1$
\ENDWHILE 
\STATE Update partition $\cbra{S_{wj}^{l}}$ of $S_{w}^{l}$ using $x$ and Alg. \ref{alg:ODA} 
\IF {Alg. \ref{alg:ODA} in $S_{w}^{l}$ terminates and $l>0$}
\STATE Split node $w$: $C(w)\gets \cbra{S_{wj}^{l}}$
\ENDIF
\UNTIL Stopping criterion 
\end{algorithmic}
\label{alg:mr-oda}
\end{algorithm}

%
%

%
\begin{figure*}[ht]
\centering
\begin{subfigure}[b]{0.31\textwidth}
\centering
\includegraphics[trim=0 0 0 0,clip,width=1.0\textwidth]{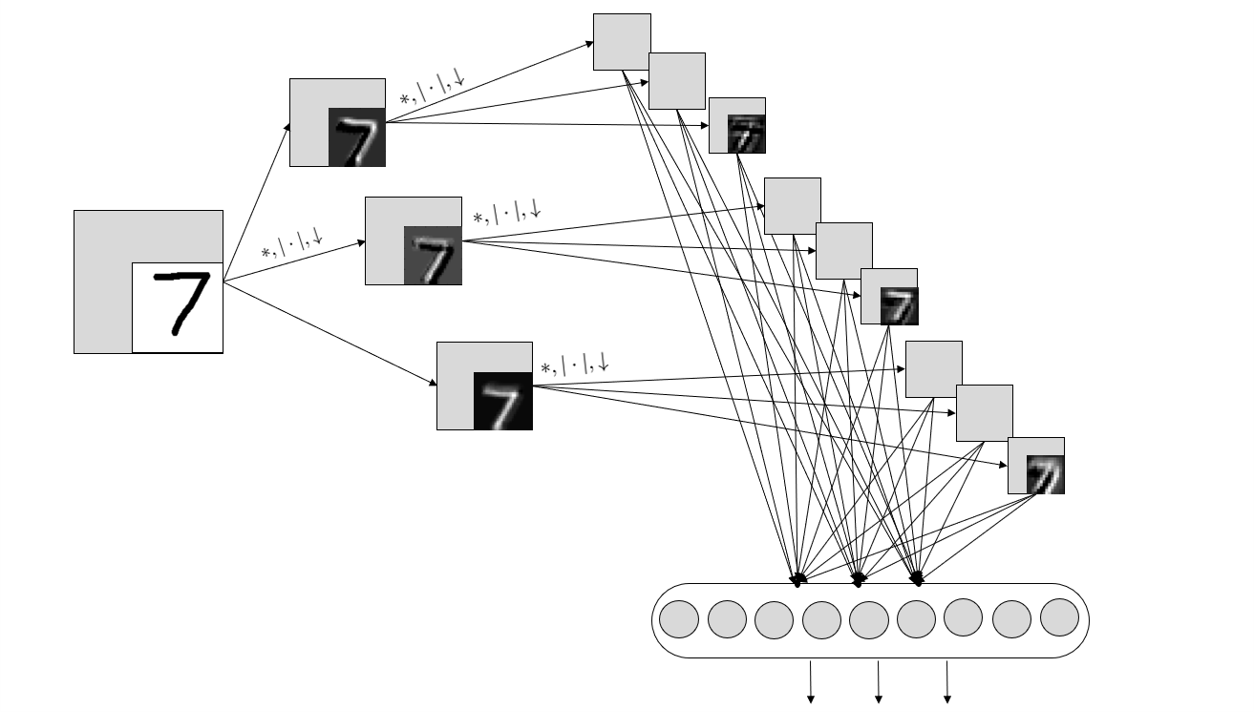}
\caption{DCNN.}
\end{subfigure}
\begin{subfigure}[b]{0.32\textwidth}
\centering
\includegraphics[trim=0 0 0 0,clip,width=1.0\textwidth]{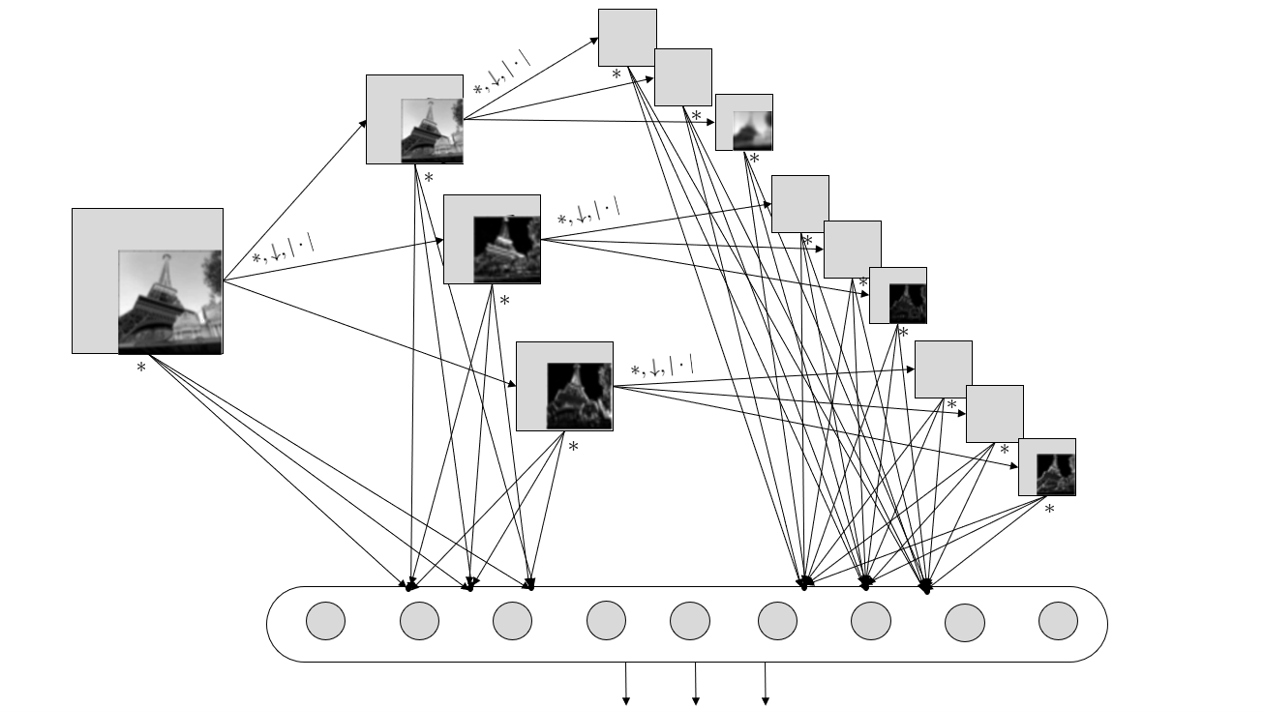}
\caption{SCN.}
\end{subfigure}
\begin{subfigure}[b]{0.35\textwidth}
\centering
\includegraphics[trim=0 0 0 0,clip,width=1.0\textwidth]{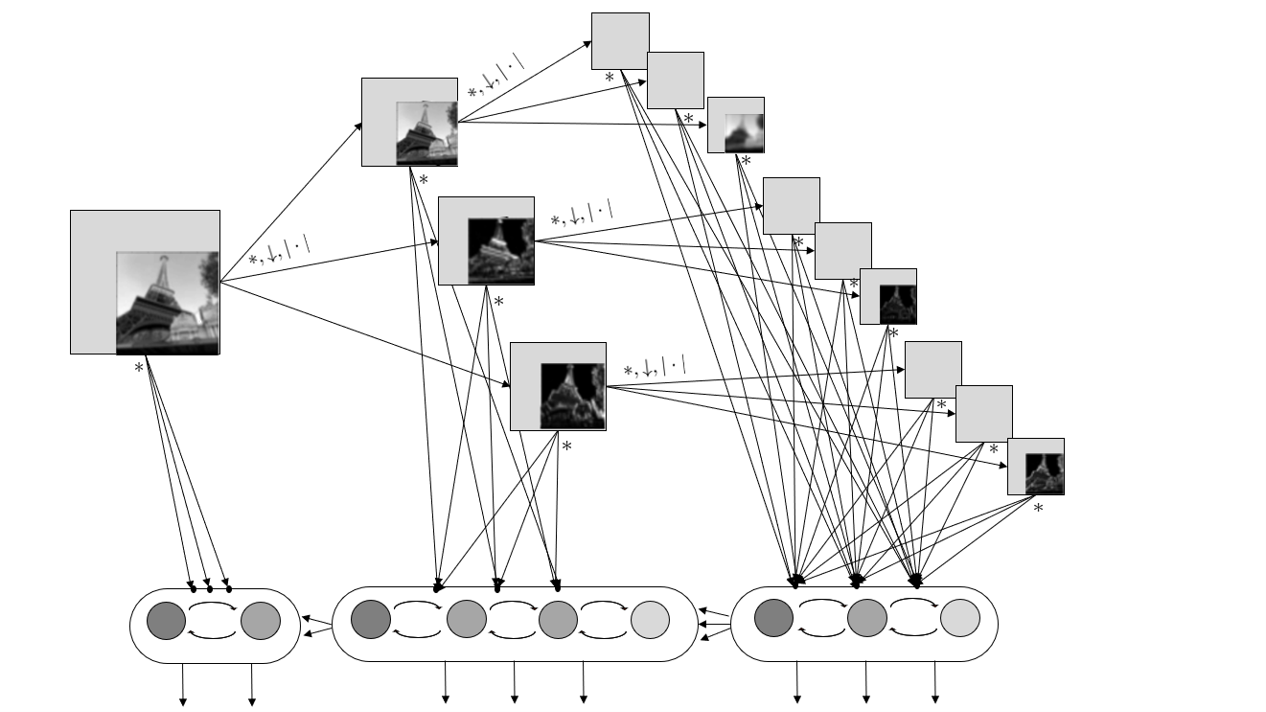}
\caption{Proposed Architecture.}
\end{subfigure}
\caption{Block-diagram of the proposed hierarchical architecture using multi-resolution features  
        from the wavelet scattering transform compared to
	Deep Convolutional Neural Networks (DCNN) and 
	Scattering Convolutional Networks (SCN).
	The feed-forward arrows represent a cascade of convolution, rectifying, and downsampling operations.}
\label{fig:TSODAvsDCNNvsSCN_intro}
\end{figure*}
%

\subsection{Building Group-Invariant Multi-Resolution Representations}
\label{sSec:scattering}

There are numerous methods to construct subspaces $\cbra{S^j}$
with the properties mentioned in Section \ref{sSec:tsoda}. 
In this section we briefly mention a particular approach based on group-convolutional wavelet decomposition
that aligns with the principles of the scattering transform, first introduced in \cite{bruna2013invariant}.
This is an unsupervised method that constructs a hierarchy of features based on group-convolutions, that 
preserve local invariance with respect to a certain class of Lie groups, such as translation, rotation, and deformation,
an important property in many learning applications
\cite{mallat2016understanding}.

We start with the standard wavelet transform $\cbra{W_l X(t)}_l$ of a signal $X(t) \in S \subseteq L^2(\mathbb{R}^d)$ as a basis \cite{mallat1999wavelet}.
Here, $\cbra{W_l X(t)}_l \in S^l$ represents the signal $X(t)$ at resolution $2^l$.
%
%
The computation of the multi-resolution wavelet representation of a signal
consists of successive operations of
a linear convolution operator,
followed by a downsampling step \cite{mallat2016understanding}.
As a result, the wavelet transform, is stable to small deformations \cite{bruna2013invariant}. 
%
In addition, the wavelet transform is translation covariant (or equivariant), that is
it commutes with the Lie group of operators 
$\cbra{T_c}_{c\in\mathbb{R}}$
such that 
$T_c X(t) = X(t-c)$, i.e.,  $W_j(T_c X) = T_c W_j(X)$.
We note that, in the 
control theory and signal processing communities,
convolutions are associated with systems described by the term 
'linear time-invariant'. 
To avoid confusion, with the terminology used here, 
these systems are considered linear covariant (or equivariant) operators  
with respect to translation in time.

To induce local invariance (up to a scale $2^J$ for some $J>0$) with respect to translation, 
it has been shown that it is sufficient to cascade the wavelet transform with a non-linear operation
$\rho W_j X= \|W_j X\|_1$,  
and a locally averaging integral operation which can be modeled as a 
convolution with a low-pass filter localized in a spatial 
window scaled at $2^J$ \cite{bruna2013invariant}.
This is called a scattering transform and its implementation 
is based on a complex-valued convolutional neural network 
whose filters are fixed wavelets and $\rho$ is a complex modulus operator
as described above \cite{andreux2019kymatio}.
%
%
This structure is 
similar to deep convolutional neural networks \cite{lecun2015deep}, 
where successive operations of
a linear convolutional operator, 
a nonlinear mapping (often a rectifying function, e.g., ReLu),
and a down-sampling step (e.g., max-pooling), are used to produce the 
input for the next stage of the architecture
\cite{bruna2013invariant,mallat2016understanding,anselmi2015deep}.
We illustrate this in Fig. \ref{fig:TSODAvsDCNNvsSCN_intro},
where Alg. \ref{alg:mr-oda} combined with the hierarchical representation of a scattering transform
is compared to a 
Deep Convolutional Network \cite{lecun2015deep} and a
Scattering Convolutional Network \cite{bruna2013invariant}.

As a final note, the translation invariance properties discussed above can be generalized to the action 
of arbitrary compact Lie groups \cite{mallat2012group}.
In particular, let $G$ be a compact Lie group and $L^2(G)$ be the space of measurable functions $f(r)$ such that 
$\|f\|^2 = \int_G |f(r)|^2 dr <\infty$, where $dr$ 
is the Haar measure of $G$. 
The left action of $g\in G$ on $f\in L^2(G)$ 
is defined by $L_g f(r) = f(g^{-1}r)$.
As a special case, the action of the translation group 
$T_c f(t) = f(t-c)$ translates the function $f$ to the right by $c$,
with $g^{-1} = -c$ translating the argument of $f$ to the left by $c$.
Similar to the usual convolution 
$(f\ast h)(x) = \int_{-\infty}^\infty f(u) h(x-u) du$
that defines a linear translation covariant operator,  
convolutions on a group appear naturally as 
linear operators covariant to the action of a group:
\begin{equation}
(f\ast h)(x) = \int_G f(g) h(g^{-1}x) dr
\end{equation}
where $dr$ is the Haar measure of $G$. 
As a result, an invariant representation 
relatively to the action of a compact Lie group, 
can be computed by averaging over covariant representations 
created by group convolution with appropriately defined wavelets, 
similar to the methodology explained above.

\section{Experimental Evaluation and Discussion}
\label{Sec:Results}

We illustrate the properties and evaluate the performance 
of the proposed learning algorithm in clustering, classification, and regression problems.

In Fig. \ref{fig:my2d-u-1}, the evolution of the 
progressive partitioning algorithm (Alg. \ref{alg:ODA}) studied in Section \ref{Sec:ODA} 
is depicted, in an unsupervised learning (clustering) problem.
To better illustrate the properties of the approach, 
the data samples were sampled from a mixture of 2D 
Gaussian distributions.
The temperature level (we use $T$ instead of $\lambda$ to stress the connection to the temperature level in annealing optimization),
the average distortion of the model, 
the number of codevectors (neurons) used, 
the number of observations (data samples) used for convergence, as well as the overall time,
are shown.
This process showcases the performance-complexity trade-off described in 
Section \ref{Sec:ODA}.
In Fig. \ref{fig:my2d-u-11} and \ref{fig:my2d-u-111}, the tree-structured progressive partitioning
algorithm of Section \ref{sSec:tsoda} is compared against Alg. \ref{alg:ODA} in the same 
problem as in Fig. \ref{fig:my2d-u-1}.
Notice that the time complexity of the algorithm is drastically reduced.
Additional properties regarding the construction of tree-structured partitions are discussed 
in Section \ref{sSec:localization}.

\begin{figure*}[h]
\centering
\begin{subfigure}[b]{0.98\textwidth}
\centering
\includegraphics[trim=0 0 0 0,clip,width=0.19\textwidth]{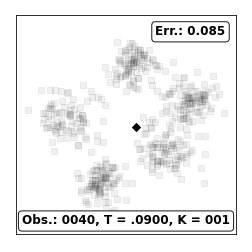}
\includegraphics[trim=0 0 0 0,clip,width=0.19\textwidth]{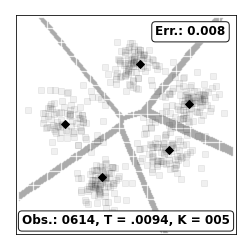}
\includegraphics[trim=0 0 0 0,clip,width=0.19\textwidth]{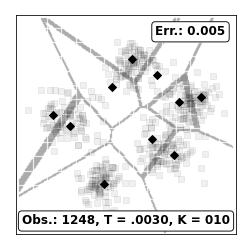}
\includegraphics[trim=0 0 0 0,clip,width=0.19\textwidth]{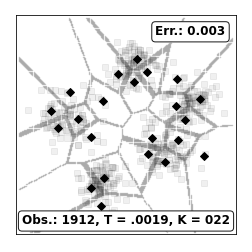}
\includegraphics[trim=0 0 0 0,clip,width=0.19\textwidth]{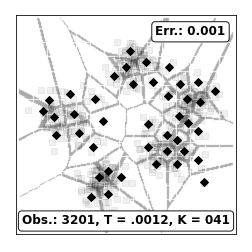}
\caption{Evolution of the algorithm in the data space.}
\label{sfig:my2d-u-1}
\end{subfigure}
\begin{subfigure}[b]{0.98\textwidth}
\centering
\includegraphics[trim=0 0 0 0,clip,width=0.24\textwidth]{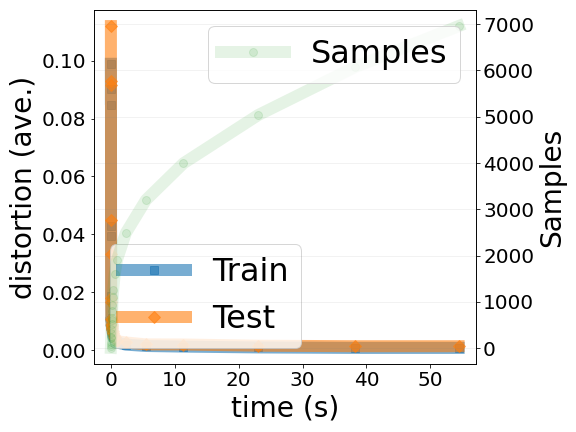}
\includegraphics[trim=0 0 0 0,clip,width=0.24\textwidth]{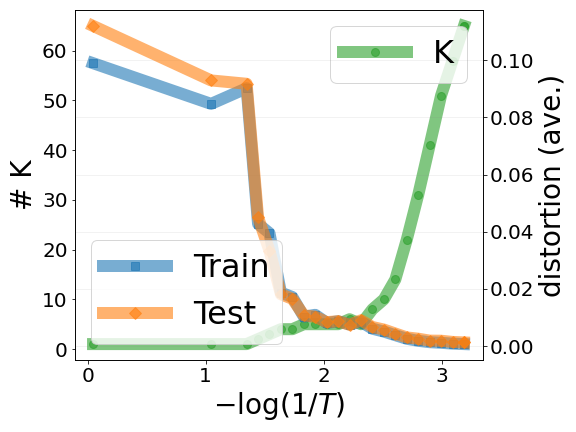}
\caption{Performance curves.}
\label{sfig:curves-my2d-u-1}
\end{subfigure}
\caption{Performance curves and data space evolution of Algorithm \ref{alg:ODA} applied to 
a clustering problem with underlying Gaussian distributions.}
\label{fig:my2d-u-1}
\end{figure*}
\begin{figure*}[h]
\centering
\begin{subfigure}[b]{0.98\textwidth}
\centering
\includegraphics[trim=0 0 0 0,clip,width=0.19\textwidth]{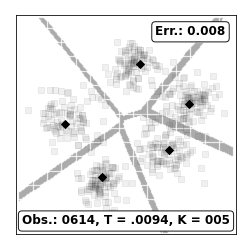}
\includegraphics[trim=0 0 0 0,clip,width=0.19\textwidth]{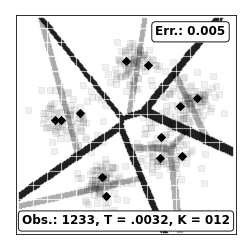}
\includegraphics[trim=0 0 0 0,clip,width=0.19\textwidth]{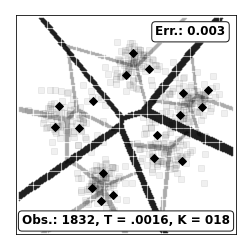}
\includegraphics[trim=0 0 0 0,clip,width=0.19\textwidth]{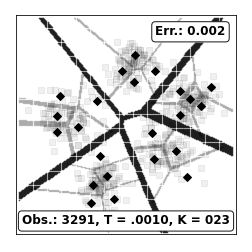}
\includegraphics[trim=0 0 0 0,clip,width=0.19\textwidth]{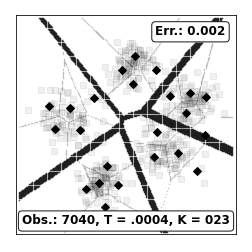}
\caption{Evolution of the algorithm in the data space.}
\label{sfig:my2d-u-11}
\end{subfigure}
\begin{subfigure}[b]{0.98\textwidth}
\centering
\includegraphics[trim=0 0 0 0,clip,width=0.24\textwidth]{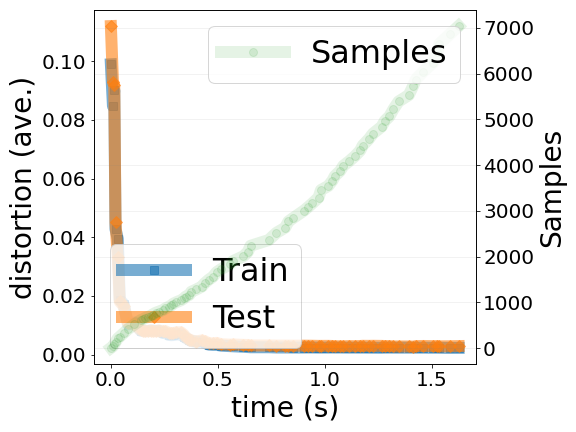}
\includegraphics[trim=0 0 0 0,clip,width=0.24\textwidth]{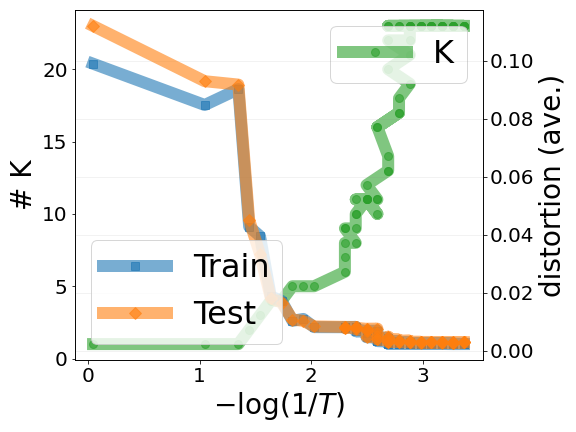}
\caption{Performance curves.}
\label{sfig:curves-my2d-u-11}
\end{subfigure}
\caption{Performance curves and data space evolution of the tree-structured approach (two layers)
applied to a clustering problem with underlying Gaussian distributions.}
\label{fig:my2d-u-11}
\end{figure*}
\begin{figure*}[h]
\centering
\begin{subfigure}[b]{0.54\textwidth}
\centering
\includegraphics[trim=0 0 0 0,clip,width=0.32\textwidth]{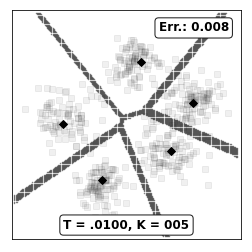}
\includegraphics[trim=0 0 0 0,clip,width=0.32\textwidth]{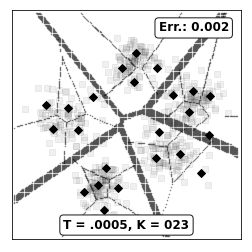}
\includegraphics[trim=0 0 0 0,clip,width=0.32\textwidth]{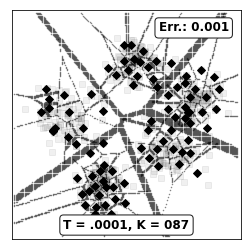}
\caption{Evolution of the algorithm in the data space.}
\label{sfig:my2d-u-111}
\end{subfigure}
\begin{subfigure}[b]{0.45\textwidth}
\centering
\includegraphics[trim=0 0 0 0,clip,width=0.49\textwidth]{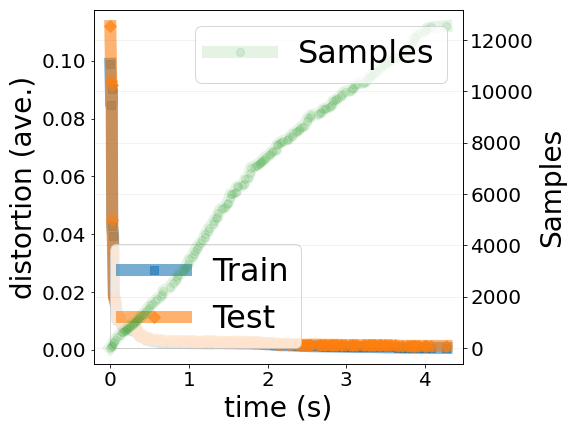}
\includegraphics[trim=0 0 0 0,clip,width=0.49\textwidth]{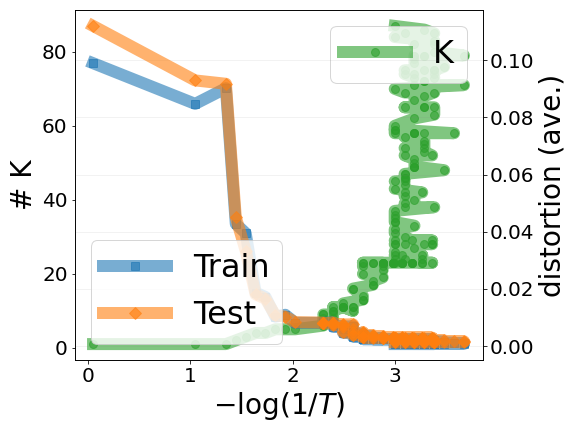}
\caption{Performance curves.}
\label{sfig:curves-my2d-u-111}
\end{subfigure}
\caption{Performance curves and data space evolution of the tree-structured approach (three layers)
applied to a clustering problem with underlying Gaussian distributions.}
\label{fig:my2d-u-111}
\end{figure*}

Similarly, Fig. \ref{fig:my2d-c-1} shows the evolution of the learning model for a 
2D classification problem with class-conditional distributions given 
by a mixture of 2D Gaussians.
These results correspond to the approach explained in Section \ref{sSec:classification},
using Alg. \ref{alg:ODA-class}.
%
In addition to the apparent accuracy-complexity trade-off, 
we make use of this classification problem to showcase the difference
of using the tree-structured approach of Section \ref{sSec:tsoda}, in terms of computational complexity.
In Fig. \ref{fig:my2d-c-111}, 
we illustrate the evolution of the 
tree-structured approach (Alg. \ref{alg:ts-oda}).
%
There are two notable comments on the behavior of this approach compared to the original.
First, the time complexity is considerably improved (see Section \ref{Sec:MR-ODA}),
and this results in a drastic difference in the running time of the learning algorithm.
%
Secondly, the number of codevectors used is drastically reduced, as well.
This is due to the regularization mechanism described in Section \ref{sSec:tsoda}:
when a partition $\cbra{S_{wj}}_{j=1}^{K_w}$ of a node $w$ is fixed, the node $w$ can check the condition $c_{\mu_{wi}} = c_{\mu_{wj}}, \forall i,j$, which means that the partition $\cbra{S_{wj}}_{j=1}^{K_w}$ is using $K_w$ codevectors, all of which
correspond to the same class. 
In this case, node $w$ is assigned only a single codevector, 
and is not further split by the algorithm.
Notice, that, in this way, the codevectors created by the tree-structured algorithm 
tend to exist in the boundaries of the Bayes decision surface, instead of populating
areas where the decision surface does not fluctuate at all.
\begin{figure*}[h]
\centering
\begin{subfigure}[b]{0.98\textwidth}
\centering
\includegraphics[trim=0 0 0 0,clip,width=0.19\textwidth]{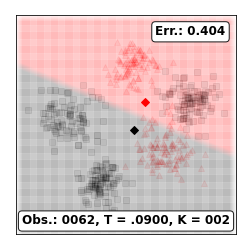}
\includegraphics[trim=0 0 0 0,clip,width=0.19\textwidth]{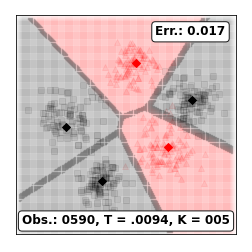}
\includegraphics[trim=0 0 0 0,clip,width=0.19\textwidth]{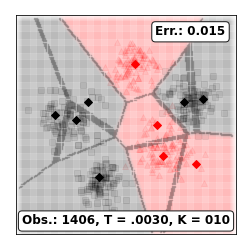}
\includegraphics[trim=0 0 0 0,clip,width=0.19\textwidth]{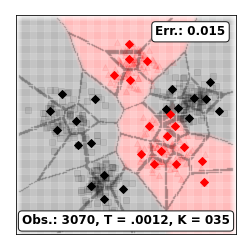}
\includegraphics[trim=0 0 0 0,clip,width=0.19\textwidth]{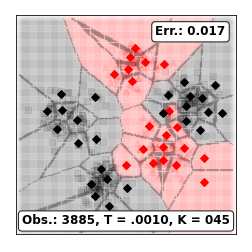}
\caption{Evolution of the algorithm in the data space.}
\label{sfig:my2d-c-1}
\end{subfigure}
\begin{subfigure}[b]{0.98\textwidth}
\centering
\includegraphics[trim=0 0 0 0,clip,width=0.24\textwidth]{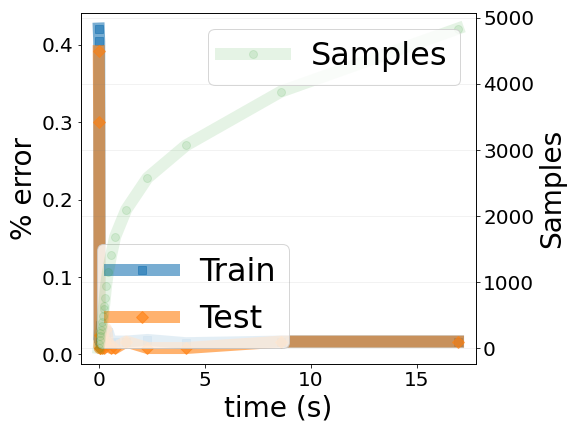}
\includegraphics[trim=0 0 0 0,clip,width=0.24\textwidth]{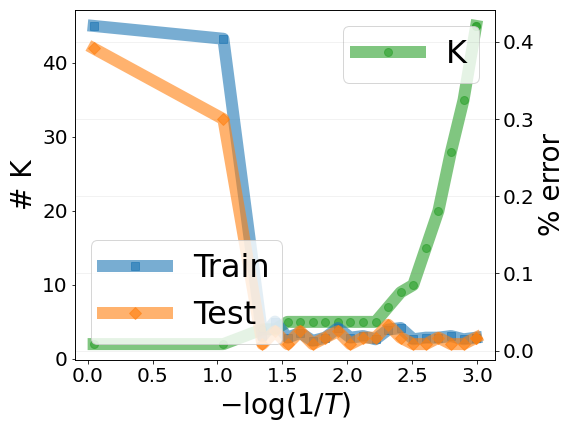}
\caption{Performance curves.}
\label{sfig:curves-my2d-c-1}
\end{subfigure}
\caption{Performance curves and data space evolution of the proposed algorithm 
applied to a classification problem with underlying Gaussian distributions.}
\label{fig:my2d-c-1}
\end{figure*}
\begin{figure*}[h]
\centering
\begin{subfigure}[b]{0.65\textwidth}
\centering
\includegraphics[trim=0 0 0 0,clip,width=0.3\textwidth]{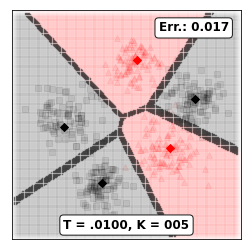}
\includegraphics[trim=0 0 0 0,clip,width=0.3\textwidth]{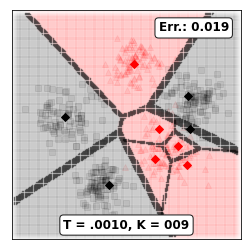}
\includegraphics[trim=0 0 0 0,clip,width=0.3\textwidth]{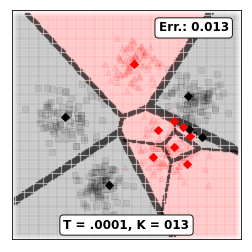}
\caption{Evolution of the algorithm in the data space.}
\label{sfig:my2d-c-111}
\end{subfigure}
\hspace{1em}
\begin{subfigure}[b]{0.26\textwidth}
\centering
\includegraphics[trim=0 0 0 0,clip,width=0.98\textwidth]{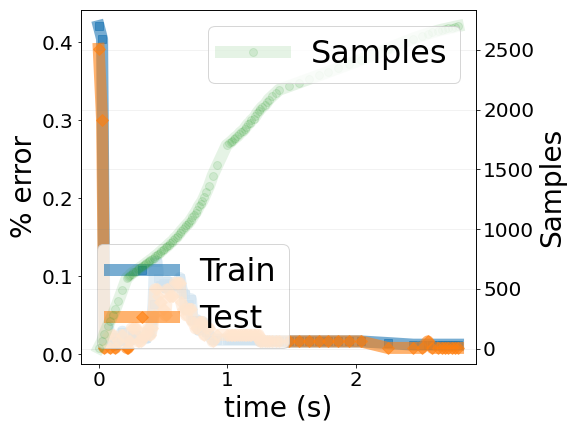}
\caption{Performance curves.}
\label{sfig:curves-my2d-c-111}
\end{subfigure}
\caption{Performance curves and data space evolution of the proposed tree-structured algorithm 
(three layers) applied to a classification problem with underlying Gaussian distributions.}
\label{fig:my2d-c-111}
\end{figure*}

The effect of using multiple resolutions
as described in Section \ref{sSec:mroda}, is depicted in Fig. \ref{fig:my2d-c-11-multi}
for the same problem as in Fig. \ref{fig:my2d-c-1} and \ref{fig:my2d-c-111}.
For better visualization, we assume that the low-resolution features, with respect to which 
the first layer of the tree is computed, are the projections of the two-dimensional data
in an one-dimensional space (line).
The second layer of the tree is trained using the high-resolution features, i.e., 
the full knowledge of both coordinates of the data.
Notice that, as expected from Theorem \ref{thm:mr-consistency}, 
this process will converge to a consistent learning algorithm, as long as the 
multi-resolution representation used complies with the properties mentioned in 
Section \ref{sSec:mroda}, and the last layer of 
the tree uses the full knowledge of the input data.

\begin{figure}[h]
\centering
\begin{subfigure}[b]{0.20\textwidth}
\centering
\includegraphics[trim=0 0 0 0,clip,width=0.98\textwidth]{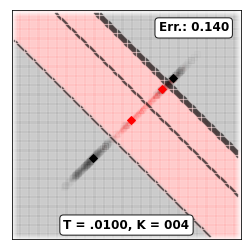}
\caption{Convergence of first layer with low-resolution features.}
\label{sfig:my2d-c-11-multi-0}
\end{subfigure}
\begin{subfigure}[b]{0.20\textwidth}
\centering
\includegraphics[trim=0 0 0 0,clip,width=0.98\textwidth]{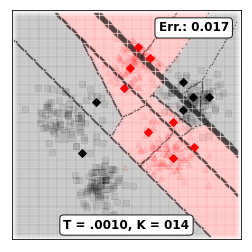}
\caption{Convergence of second layer with high-resolution features.}
\label{sfig:my2d-c-11-multi-1}
\end{subfigure}
\begin{subfigure}[b]{0.25\textwidth}
\centering
\includegraphics[trim=0 0 0 0,clip,width=0.98\textwidth]{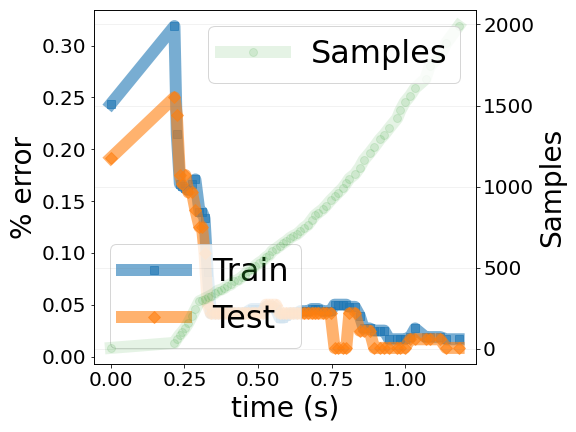}
\caption{Performance curve.}
\label{sfig:curves-my2d-c-11-multi}
\end{subfigure}
\caption{Performance curves and data space evolution of the proposed multi-resolution algorithm 
(two layers) applied to a classification problem with underlying Gaussian distributions.}
\label{fig:my2d-c-11-multi}
\end{figure}

Finally, in Fig. \ref{fig:r1d} and \ref{fig:r2d}, 
we test the proposed methodology in two regression problems,
where one- and two-dimensional functions are hierarchically approximated 
using the piece-wise constant approximation algorithm of Section \ref{sSec:PWC-regression}, and the
tree-structured approach of Section \ref{Sec:MR-ODA}.

\begin{figure*}[h]
\centering
\begin{subfigure}[b]{0.98\textwidth}
\centering
\includegraphics[trim=0 0 0 0,clip,width=0.22\textwidth]{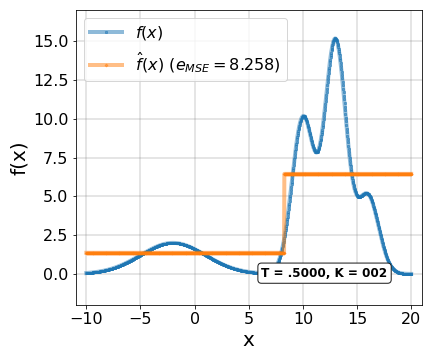}
\includegraphics[trim=0 0 0 0,clip,width=0.22\textwidth]{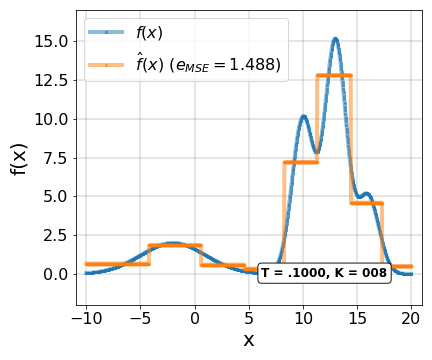}
\includegraphics[trim=0 0 0 0,clip,width=0.22\textwidth]{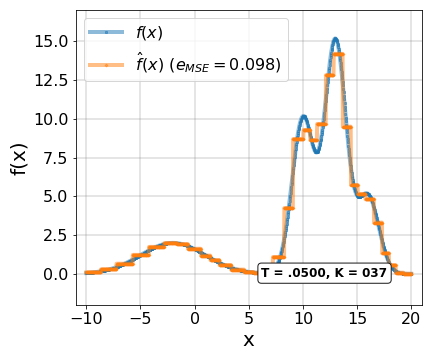}
\includegraphics[trim=0 0 0 0,clip,width=0.22\textwidth]{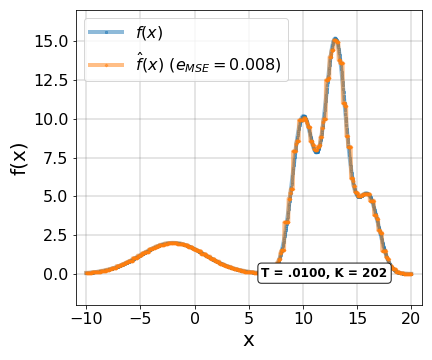}
\caption{Evolution of the algorithm in the data space.}
\label{sfig:r1d}
\end{subfigure}
\begin{subfigure}[b]{0.98\textwidth}
\centering
\includegraphics[trim=0 0 0 0,clip,width=0.23\textwidth]{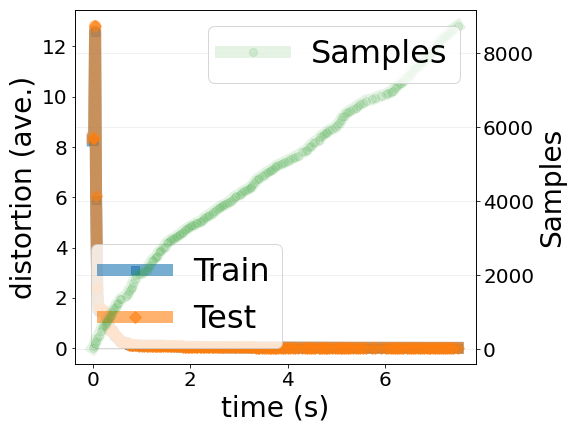}
\includegraphics[trim=0 0 0 0,clip,width=0.23\textwidth]{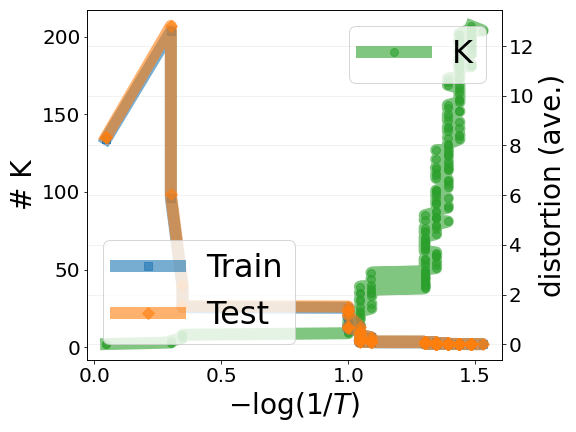}
\caption{Performance curves.}
\label{sfig:curves-r1d}
\end{subfigure}
\caption{Performance curves and data space evolution of the proposed tree-structured algorithm 
(four layers) applied to a piece-wise constant function approximation problem in 1D.}
\label{fig:r1d}
\end{figure*}
\begin{figure*}[h]
\centering
\begin{subfigure}[b]{0.76\textwidth}
\centering
\includegraphics[trim=120 20 80 50,clip,width=0.24\textwidth]{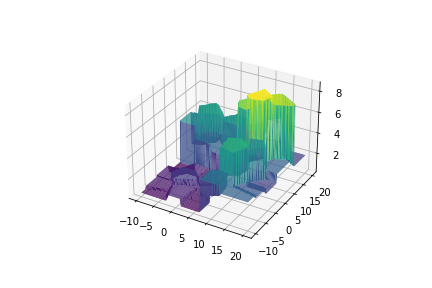}
\includegraphics[trim=120 20 80 50,clip,width=0.24\textwidth]{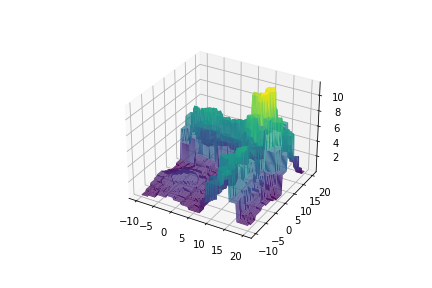}
\includegraphics[trim=120 20 80 50,clip,width=0.24\textwidth]{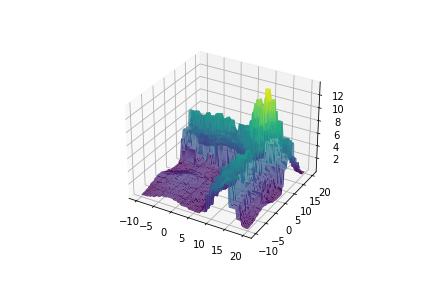}
\includegraphics[trim=120 20 80 50,clip,width=0.24\textwidth]{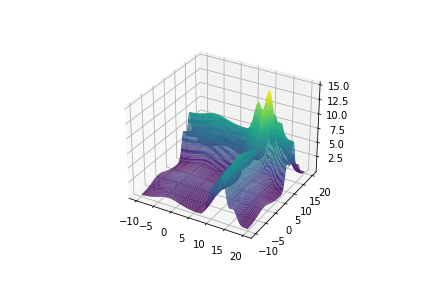}
\caption{Evolution of the algorithm in the data space (original function on the right).}
\label{sfig:r2d}
\end{subfigure}
\begin{subfigure}[b]{0.22\textwidth}
\centering
\includegraphics[trim=0 0 0 0,clip,width=0.999\textwidth]{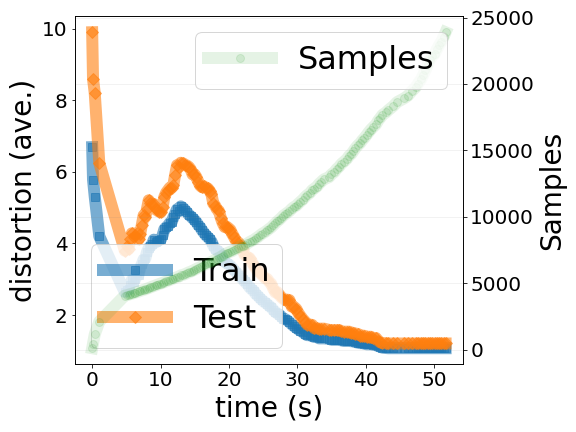}
\caption{Performance curves.}
\label{sfig:curves-r2d}
\end{subfigure}
\caption{Performance curves and data space evolution of the proposed tree-structured algorithm 
(three layers) applied to a piece-wise constant function approximation problem in 2D.}
\label{fig:r2d}
\end{figure*}
%

\subsection{Source Code and Reproducibility}

The open-source code is publicly available 
at \url{https://github.com/MavridisChristos/OnlineDeterministicAnnealing}.

\subsection{Tree-Structured Partition, Localization, and Explainability in Machine Learning}
\label{sSec:localization}

Another advantage of using a tree-structured learning module is the localization properties
which allow for an understanding of the input space, in accordance to the principles of 
the recently intoduced class of explainable learning models \cite{milani2022survey}. 
%
The Voronoi regions shrink geometrically, 
and allow for the use of local models, which is especially important in high-dimensional spaces. 
Unlike most learning models, it is possible to locate the area of the data space that 
presents the highest error rate and selectively split it by using local ODA.
This process can be iterated until the desired error rate (or average distortion) is achieved.
When using a training dataset for classification, it is often possible to force 
accuracy of up to $100\%$ 
on the training dataset. 
This is similar to an over-fitted classification and regression tree (CART)
\cite{breiman2001random}.
However, over-fitting on the training dataset often adversely affects the 
generalization properties of the model, the performance on the 
testing dataset, and the robustness against adversarial attacks. 
Therefore, the progressive process of ODA becomes important in establishing a robust way to 
control the trade-off between performance and complexity, before you reach that limit.
Finally, an important question in tree-structured learning models is 
the question of which cell to split next.
An exhaustive search in the entire tree to find the node that presents the largest error rate 
is possible but is often not desired due to the large computational overhead. 
This is automatically answered by the 
multi-resolution ODA algorithm (Alg. \ref{alg:mr-oda}) as it asynchronously updates all cells 
depending on the sequence of the online observations. 
As a result, the regions of the data space that are more densely populated with data samples 
are trained first, which results in a higher percentage of performance increase per cell split.
We stress that this property makes the proposed algorithm completely dataset-agnostic, 
in the sense that it does not require the knowledge of 
a training dataset a priori, but instead operates completely online, i.e., using one observation at a time 
to update its knowledge base.

\section{Conclusion}
\label{Sec:Conclusion}

We introduced a hierarchical learning algorithm 
to gradually approximate a solution to a data-driven optimization problem
in the context of autonomous decision-making systems, especially under limitations on time and computational resources.
The learning architecture simulates an annealing process 
and defines a heuristic method 
to progressively construct a tree-structured partition
of a possibly multi-resolution data space,
which can be used in conjunction with general learning algorithms to train local models. 
The structured partitioning of the input space provides explainability, 
and makes the learning architecture a suitable candidate for transfer learning 
applications. 
Finally, the online gradient-free training rule based on stochastic approximation, can be viewed 
a discrete-time dynamical learning system, and used for inference, control, and reinforcement 
learning applications. 

\ifx\mycmd\undefined

\bibliographystyle{IEEEtran} %
\bibliography{bib_learning,bib_mavridis} 

\else

\bibliographystyle{siamplain}
\bibliography{bib_learning,bib_mavridis}

\fi

\ifx\mycmd\undefined

\appendices

\else

\appendix

\fi

\section{Proof of Lemma \ref{lem:gibbs} (Derivation of the Association Probabilities).} 
\label{App:gibbs}

%
\noindent
Recall that by definition \eqref{eq:F}, we get 
\begin{align*}
F(\mu) &:= (1-\lambda) \int p(x) \sum_i p(\mu_i|x) d(x,\mu_i) ~dx   
    \\&\quad
    + \lambda \int p(x) \sum_i p(\mu_i|x) \log p(\mu_i|x) ~dx 
    - \lambda H(X)
\end{align*}
We form the Lagrangian:
\begin{equation}
\begin{aligned}
    \mathcal L_f&(\cbra{p(\mu_i|x)},\nu)
    := 
    \\&
    =(1-\lambda) D(\mu) -  \lambda H(\mu) 
    + \nu \pbra{\sum_i p(\mu_i|x) - 1} 
    \\&
    = (1-\lambda) \int p(x) \sum_i p(\mu_i|x) d(x,\mu_i) ~dx 
    \\&\quad    
    +  \lambda \int p(x) \sum_i p(\mu_i|x) \log p(\mu_i|x) ~dx 
    \\&\quad    
    + \nu \pbra{\sum_i p(\mu_i|x) - 1}  
    - \lambda \E{-\log p(X)} 
\end{aligned}    
\end{equation}
Taking $\frac{\partial \mathcal L}{\partial p(\mu|x) } = 0 $ yields:
\begin{align*}
    & (1-\lambda) d(x,\mu_i) + \lambda (1+\log p(\mu_i|x)) + \nu = 0 \\
    & \implies \log p(\mu_i|x) = - \frac{1-\lambda}{\lambda} d(x,\mu_i)
            - \pbra{ 1 + \frac \nu \lambda } \\
    & \implies p(\mu_i|x) =  \frac{e^{- \frac{1-\lambda}{\lambda} d(x,\mu_i)}}{e^{1 + \frac \nu \lambda}}      
\end{align*}
Finally, from the condition $\sum_i p(\mu_i|x) = 1$, it follows that 
\begin{equation*}
    e^{1 + \frac \nu \lambda} = \sum_i e^{- \frac{1-\lambda}{\lambda} d(x,\mu_i)}
\end{equation*}
which completes the proof.

\section{Proof of Theorem \ref{thm:ODA} (Convergence of the Online Learning Rule).} 
\label{App:online}

\noindent
We are going to use fundamental results from 
stochastic approximation theory.
For completeness, we present the key theorems in what follows.
\begin{theorem}[\cite{borkar2009stochastic}, Ch.2]
\label{thm:borkar}
	Almost surely, the sequence $\cbra{x_n}\in S\subseteq\mathbb{R}^d$ 
	generated by the following stochastic approximation scheme:
	\begin{align}
		x_{n+1} = x_n + \alpha(n) \sbra{h(x_n) + M_{n+1}},\ n \geq 0	
	\label{eq:sa}	
	\end{align}
	with prescribed $x_0$, 
	\textit{converges} to a (possibly sample path dependent)
	compact, connected, internally chain transitive, invariant set
	of the o.d.e:
	\begin{align}
		\dot{x}(t) = h\pbra{x(t)}, ~ t \geq 0, 	
	\label{eq:sa_ode}	
	\end{align}
	where $x:\mathbb{R}_+\rightarrow\mathbb{R}_d$ and $x(0) = x_0$, 
	provided the following assumptions hold:
	\begin{itemize}
	\setlength\itemsep{0em}
	\item[(A1)] The map $h:\mathbb{R}^d \rightarrow \mathbb{R}^d$ is Lipschitz
		in $S$,	i.e., $\exists L$ with $0 < L < \infty$ such that
		$\norm{h(x)-h(y)} \leq L\norm{x-y}, ~ x,y \in S$,
	\item[(A2)] The stepsizes $\cbra{\alpha(n) \in \mathbb{R}_{++}, ~ n \geq 0}$
	satisfy
		$ \sum_n \alpha(n) = \infty$, and $\sum_n \alpha^2(n) < \infty$	,
	\item[(A3)] $\cbra{M_n}$ is a martingale difference sequence 
		with respect to the increasing family of $\sigma$-fields
		$ \mathcal{F}_n := \sigma \pbra{ x_m, M_m,~ m \leq n }$, ${n \geq 0}$,
		i.e., $\E{M_{n+1}|\mathcal{F}_n} = 0 ~ a.s.$, for all $n \geq 0$,
		and $\cbra{M_{n}}$ are square-integrable with 
		$ \E{\norm{M_{n+1}}^2|\mathcal{F}_n} \leq K \pbra{ 1 + \norm{x_n}^2 }, 
		~ a.s.$, where $n \geq 0 $ for some $K >0$,
	\item[(A4)] The iterates $\cbra{x_n}$ remain bounded a.s., i.e.,
		${ \sup_n \norm{x_n} < \infty}$ $ a.s.$
	\end{itemize}
\end{theorem}

As an immediate result, the following corollary also holds:
\begin{corollary}
If the only internally chain transitive invariant sets for
(\ref{eq:sa_ode}) are isolated equilibrium points,
then, almost surely, $\cbra{x_n}$ converges to a, 
possibly sample dependent, equilibrium point of (\ref{eq:sa_ode}).  
\label{crl:sa_equillibria}
\end{corollary}

\noindent
Now we are in place to prove the following theorem:

\begin{theorem}
Let $S$ a vector space, $\mu\in S$, and
$X: \Omega \rightarrow S$
be a random variable defined in a 
probability space $\pbra{\Omega, \mathcal{F}, \mathbb{P}}$.
Let $\cbra{x_n}$ be a sequence of independent realizations of $X$,
and $\cbra{\alpha(n)>0}$ a sequence of stepsizes such that
$ \sum_n \alpha(n) = \infty$, and $\sum_n \alpha^2(n) < \infty$.
Then the random variable $m_n = \nicefrac{\sigma_n}{\rho_n}$,
where $(\rho_n, \sigma_n)$ are 
sequences defined by
\begin{equation}
\begin{aligned}
\rho_{n+1} &= \rho_n + \alpha(n) \sbra{ p(\mu|x_n) - \rho_n} \\
\sigma_{n+1} &= \sigma_n + \alpha(n) \sbra{ x_n p(\mu|x_n) - \sigma_n},
\end{aligned}
\label{eq:rhosigma}
\end{equation}
converges to $\E{X|\mu}$ almost surely, i.e. 
$m_n\xrightarrow{a.s.} \E{X|\mu}$.
\label{thm:oda_sa}
\end{theorem}
\begin{proof}
We will use the facts that $p(\mu)=\E{p(\mu|x)}$ and 
$\E{\mathds{1}_{\sbra{\mu}}X} = \E{xp(\mu|x)}$.
The recursive equations (\ref{eq:rhosigma}) are 
stochastic approximation algorithms of the form:
\begin{equation*}
\begin{aligned}
\rho_{n+1} &= \rho_n + \alpha(n)  
	[ (p(\mu) - \rho_n) + 
	(p(\mu|x_n)-\E{p(\mu|X)}) ] \\
\sigma_{n+1} &= \sigma_n + \alpha(n) 
	[ (\E{\mathds{1}_{\sbra{\mu}}X} - \sigma_n) + 
	\\&\quad\quad\quad\quad\quad\quad
	(x_n p(\mu|x_n) - \E{x_n p(\mu|X)})  ]
\end{aligned}
\label{eq:rhosigma_sa}
\end{equation*}
It is obvious that both stochastic approximation algorithms
satisfy the conditions of 
Theorem \ref{thm:borkar} and Corollary \ref{crl:sa_equillibria}.
As a result, they converge to the asymptotic solution of the 
differential equations
\begin{equation*}
\begin{aligned}
\dot \rho &= p(\mu) - \rho \\
\dot \sigma &= \E{\mathds{1}_{\sbra{\mu}}X} - \sigma
\end{aligned}
\end{equation*}
which can be trivially derived through standard ODE analysis to 
be $\pbra{p(\mu), \E{\mathds{1}_{\sbra{\mu}}X}}$.
In other words, we have shown that
\begin{equation*}
\pbra{\rho_n,\sigma_n} \xrightarrow{a.s.} \pbra{p(\mu), \E{\mathds{1}_{\sbra{\mu}}X}}
\end{equation*}
The convergence of $m_n$ follows from the fact that 
$\E{X|\mu} = \nicefrac{\E{\mathds{1}_{\sbra{\mu}}X}}{p(\mu)}$,
and standard results on the convergence 
of the product of two random variables.
\end{proof}

\section{Proof of Theorem \ref{thm:consistency} (Consistency of ODA as a Density Estimator).} 
\label{App:consistency}

\noindent
According to Theorem \ref{thm:ODA}, as $n\rightarrow\infty$,
the stochastic approximation 
algorithm in (\ref{eq:oda_learning1}), (\ref{eq:oda_learning2}) 
minimizes the cost function $F^*$ in (\ref{eq:minFstar}).
Moreover, it is easy to see that, in the limit $\lambda\rightarrow 0$,
we get 
\begin{equation*}
\lim_{\lambda\rightarrow 0} p^*(\mu_i|x) 
=\lim_{\lambda\rightarrow 0} \frac{e^{-\frac{1-\lambda}{\lambda}d(x,\mu_i)}}
			{\sum_j e^{-\frac{1-\lambda}{\lambda}d(x,\mu_j)}}
=\mathds{1}_{\sbra{x\in S_i}}    
\end{equation*}
and 
\begin{align*}
\lim_{\lambda\rightarrow 0} F^*(\mu) = J(\mu) &= \E{\min_i d(X,\mu_i)} 
\\&
= \int p(x) \sum_i \mathds{1}_{\sbra{x\in S_i}}  d_\phi(x,\mu_i) ~dx
\\&
= \sum_i \int_{S_i} p(x)  d_\phi(x,\mu_i) ~dx
\end{align*}
where 
$S_i = \cbra{x \in S: i = \argmin\limits_j ~ d(x,\mu_j)}$.
In addition, due to the bifurcation phenomenon, 
$\lambda\rightarrow 0$, induces $k\rightarrow\infty$.

Next we show that as the number of
prototypes goes to infinity, i.e., 
if $k\rightarrow\infty$, 
we get $\min_\mu J(\mu)=0$.
First, consider a sub-optimal solution $w:=\cbra{w_j}_{j=1}^k$, 
with $p(w_i|x) = \mathds{1}_{\sbra{x\in \Sigma_i}}$
with the property that $Vol(\Sigma_i)=\int_{\Sigma_i} dx = O(\frac{1}{k})$, i.e., 
the Voronoi cells $\Sigma_i(k)$ form a roughly uniform partition. 
In that case  
\begin{equation*}
\lim_{k\rightarrow\infty} J(w) 
=  \lim_{k\rightarrow\infty}  \sum_i \int_{\Sigma_i} p(x)  d_\phi(x, w_i) ~dx 
= 0
\end{equation*}
where we have used the continuity of the density, the compactness of $S$, 
and the fact that $\lim_{k\rightarrow\infty} Vol(\Sigma_j)=0$ since
$Vol(\Sigma_j)=O(\frac{1}{k})$.
We note that these convergence results hold as long as 
$\frac{k}{n}\rightarrow 0$, i.e., 
the rate of increase of $k$ is lower than that of the number of observations $n$.

As a result, 
since $0\leq J(\mu) \leq J(w) \rightarrow 0$,
due to the optimality of $\mu$,
it follows that $J(\mu) \rightarrow 0$ a.s., as well.
This implies that $\lim_{k\rightarrow\infty} Vol(S_i)\rightarrow 0$. 
Now define the random variable
\begin{equation*}
Y_n:= \frac{\mathds{1}_{\sbra{x_n\in S_i}}}{Vol(S_i)}    
\end{equation*}
where 
\begin{equation}
\E{Y_n} := \frac{\E{\mathds{1}_{\sbra{x_n\in S_i}}}}{Vol(S_i)}
= \frac{\mathbb{P}\sbra{x_n\in S_i}}{Vol(S_i)}
= \frac{\int_{x\in S_i} p(x) ~dx}{\int_{x\in S_i} ~dx}
= \bar Y 
\label{eq:barY}
\end{equation}
From the Strong Law of Large Numbers (SLLN), we get that
\begin{equation*}
\frac 1 n \frac{\sum_n \mathds{1}_{\sbra{x_n\in S_i}}}{Vol(S_i)}  \rightarrow
\bar Y,\ a.s.
\end{equation*}
The claim that for $n\rightarrow\infty$, $k\rightarrow\infty$, and 
$\frac{k}{n}\rightarrow 0$, 
\begin{equation*}
\hat p(x):= \frac 1 n \frac{\sum_n \mathds{1}_{\sbra{x_n\in S_i}}}{Vol(S_i)}  \rightarrow
p(x),\ a.s.
\end{equation*}
follows by observing that $\lim_{k\rightarrow\infty} \bar Y = p(x)$ which follows 
from \eqref{eq:barY} and the fact that the limit $\lim_{k\rightarrow\infty} Vol(S_i)\rightarrow 0,\ \forall i$.

\section{Proof of Theorem \ref{thm:bifurcation} (Proof of Bifurcation Phenomena).} 
\label{App:bifurcation}

\noindent
To obtain the optimality condition \eqref{eq:soc} as a function of 
the temperature level $\lambda$, we recall that 
\begin{align*}
&F^*(y) = (1-\lambda) \int p(x) \sum_i p(y_i|x) d_\phi(x,y_i) ~dx 
    \\&\quad
                + \lambda \int p(x) \sum_i p(y_i|x) \log p(y_i|x) ~dx 
                + \lambda \int p(x) \log p(X)
\end{align*}
where
$y = \mu + \epsilon \psi$, and $d_\phi$ is a Bregman divergence for an appropriately
defined strictly convex function $\phi$.
By direct differentiation, we can compute
\begin{equation}
    \frac{d}{d\epsilon} d_\phi(x,y_i) = - \pder{^2 \phi(y_i)}{y_i^2}(x-y_i)^\T \psi 
    \label{eq:bifurcation:dd}
\end{equation}
and
\begin{equation}
    \frac{d^2}{d\epsilon^2} d_\phi(x,y_i) = \pder{^2 \phi(y_i)}{y_i^2} \psi^\T \psi    
    \label{eq:bifurcation:d2d}
\end{equation}
and given (\ref{eq:gibbs}), we get
\begin{equation}
\begin{aligned}
\frac{d}{d\epsilon}& p(y_i|x) = \frac{1-\lambda}{\lambda} 
        \sum_j p(y_i|x) p(y_j|x)  
    \bigg[\pder{^2 \phi(y_i)}{y_i^2} (x-y_i)^\T \psi_i -
    \\&\quad\quad\quad\quad\quad\quad\quad\quad\quad\quad\quad\quad\quad\quad\quad
         \pder{^2 \phi(y_j)}{y_j^2}(x-y_j)^\T \psi_j\bigg]
    \\&= \frac{1-\lambda}{\lambda} p(y_i|x) \pder{^2 \phi(y_i)}{y_i^2} (x-y_i)^\T \psi_i 
    \\&
         -  \frac{1-\lambda}{\lambda} p(y_i|x) \sum_{j} p(y_j|x)
            \pder{^2 \phi(y_j)}{y_j^2} (x-y_j)^\T \psi_j
    \\&= - \frac{1-\lambda}{\lambda} p(y_i|x) \frac{d}{d\epsilon} d_\phi(x,y_i)
    \\&
         + \frac{1-\lambda}{\lambda} p(y_i|x) \sum_{j} p(y_j|x)
            \frac{d}{d\epsilon} d_\phi(x,y_j)      
\end{aligned}
\label{eq:bifurcation:dp}
\end{equation}
Now the optimality condition takes the form:
\begin{equation}
\begin{aligned}
\frac{d^2}{d\epsilon^2} F^*(y) &= 
(1-\lambda) \int p(x) \sum_i \frac{d^2}{d\epsilon^2} \pbra{ p(y_i|x) d_\phi(x,y_i) } ~dx
\\&
+ \lambda \int p(x) \sum_i \frac{d^2}{d\epsilon^2} \pbra{ p(y_i|x) \log p(y_i|x) } ~dx 
\end{aligned}
    \label{eq:bifurcation:d2Fstar}
\end{equation}
Equation (\ref{eq:bifurcation:d2Fstar}) uses the terms given in 
\eqref{eq:bifurcation:d2pd}, and \eqref{eq:bifurcation:d2plogp} below: 
\begin{equation}
\begin{aligned}
    \frac{d^2}{d\epsilon^2}&\pbra{p(y_i|x) d_\phi(x,y_i)} = 
    p(y_i|x) \frac{d^2}{d\epsilon^2} d_\phi(x,y_i) 
    \\&\quad +2 \frac{d}{d\epsilon} d_\phi(x,y_i) \frac{d}{d\epsilon} p(y_i|x) 
    + d_\phi(x,y_i) \frac{d^2}{d\epsilon^2} p(y_i|x) 
\end{aligned}
\label{eq:bifurcation:d2pd}
\end{equation}
%
%
\begin{equation}
\begin{aligned}
    \frac{d^2}{d\epsilon^2}&\pbra{p(y_i|x) \log p(y_i|x)} = 
    \frac{d^2}{d\epsilon^2} p(y_i|x)
    \\&
    + \log p(y_i|x) \frac{d^2}{d\epsilon^2} p(y_i|x) 
    + \frac{1}{p(y_i|x)} \pbra{\frac{d}{d\epsilon} p(y_i|x)}^2
\end{aligned}
    \label{eq:bifurcation:d2plogp}
\end{equation}
First, notice that $\sum_i \frac{d^2}{d\epsilon^2} p(y_i|x) =  \frac{d^2}{d\epsilon^2} \sum_i p(y_i|x) = 0$.
Using the expressions 
(\ref{eq:bifurcation:dd}), (\ref{eq:bifurcation:d2d}), and (\ref{eq:bifurcation:dp})
we get \eqref{eq:biftemp1}, \eqref{eq:biftemp2},
\eqref{eq:biftemp3}, \eqref{eq:biftemp4} that read as:
\begin{equation}
    \sum_i p(y_i|x) \frac{d^2}{d\epsilon^2} d_\phi(x,y_i) = 
    \sum_i p(y_i|x) \pder{^2 \phi(y_i)}{y_i^2} \psi_i^\T \psi_i 
    \label{eq:biftemp1}
\end{equation}
%
%
\small{
\begin{equation}
    \begin{aligned}
    \sum_i 2 &\frac{d}{d\epsilon} d_\phi(x,y_i) \frac{d}{d\epsilon} p(y_i|x) = 
    \\&
    =-2 \frac{1-\lambda}{\lambda} \sum_i p(y_i|x) \pbra{ \pder{^2 \phi(y_i)}{y_i^2} (x-y_i)^\T \psi_i}^2  
    \\& 
    + 2 \frac{1-\lambda}{\lambda} \sum_i p(y_i|x) \pder{^2 \phi(y_i)}{y_i^2} (x-y_i)^\T \psi_i
    \\& \quad\quad\quad\quad\quad
    \sum_{j} p(y_j|x) \pder{^2 \phi(y_j)}{y_j^2} (x-y_j)^\T \psi_j
    \\&=
    \sum_i p(y_i|x) \pbra{\pder{^2 \phi(y_i)}{y_i^2}}^2 
    \\&\quad\quad
    \psi_i^\T \sbra{ - 2 \frac{1-\lambda}{\lambda}  (x-y_i) (x-y_i)^\T } \psi_i  
    \\&\quad
    + 2 \frac{1-\lambda}{\lambda} \pbra{\sum_i p(y_i|x) \pder{^2 \phi(y_i)}{y_i^2} (x-y_i)^\T \psi_i}^2
\end{aligned}
\label{eq:biftemp2}
\end{equation}
}
%
%
%
%
\begin{equation}
    \begin{aligned}
    \sum_i &\log p(y_i|x) \frac{d^2}{d\epsilon^2} p(y_i|x) = 
    \\&
    =- \sum_i \frac{1-\lambda}{\lambda} d_\phi(x,y_i) \frac{d^2}{d\epsilon^2} p(y_i|x)
    \\&
    - \sum_i \log\pbra{\sum_{j} e^{-\frac{1-\lambda}{\lambda}d(x,y_j)}} \frac{d^2}{d\epsilon^2} p(y_i|x)
    \\&= 
    - \sum_i \frac{1-\lambda}{\lambda} d_\phi(x,y_i) \frac{d^2}{d\epsilon^2} p(y_i|x)
    \\&
    - \pbra{\log\sum_{j} e^{-\frac{1-\lambda}{\lambda}d(x,y_j)}} \sum_i  \frac{d^2}{d\epsilon^2} p(y_i|x)
    \\&= 
    - \frac{1-\lambda}{\lambda} \sum_i  d_\phi(x,y_i) \frac{d^2}{d\epsilon^2} p(y_i|x)
\end{aligned}
\label{eq:biftemp3}
\end{equation}
%
%
\begin{equation}
    \begin{aligned}
    \sum_i &\frac{1}{p(y_i|x)} \pbra{\frac{d}{d\epsilon} p(y_i|x)}^2 =
    \\& 
    = \frac{\pbra{1-\lambda}^2}{\lambda^2} \sum_i p(y_i|x) \pbra{\frac{d}{d\epsilon} d_\phi(x,y_i)}^2
    \\& 
    + \frac{\pbra{1-\lambda}^2}{\lambda^2} \sum_i p(y_i|x) \pbra{\sum_{y_j} p(y_j|x) \frac{d}{d\epsilon} d_\phi(x,y_j)}^2
    \\&
    - 2 \frac{\pbra{1-\lambda}^2}{\lambda^2} \pbra{ \sum_i p(y_i|x) \frac{d}{d\epsilon} d_\phi(x,y_i)}^2
    \\&=
    \frac{\pbra{1-\lambda}^2}{\lambda^2} \sum_i p(y_i|x) \pbra{\frac{d}{d\epsilon} d_\phi(x,y_i)}^2
    \\& 
    - \frac{\pbra{1-\lambda}^2}{\lambda^2} \pbra{ \sum_i p(y_i|x) \frac{d}{d\epsilon} d_\phi(x,y_i)}^2
\end{aligned}
\label{eq:biftemp4}
\end{equation}
%
Finally, plugging \eqref{eq:biftemp1}, \eqref{eq:biftemp2},
\eqref{eq:biftemp3}, \eqref{eq:biftemp4} in \eqref{eq:bifurcation:d2Fstar}, the optimality condition (\ref{eq:soc}) becomes
%
\begin{equation*}
    \begin{aligned}
    0&=(1-\lambda) \int p(x)  \sum_i p(y_i|x) \frac{d^2}{d\epsilon^2} d_\phi(x,y_i) dx
    \\& - 
    \frac{\pbra{1-\lambda}^2}{\lambda} \int p(x)  \sum_i p(y_i|x) \pbra{\frac{d}{d\epsilon} d_\phi(x,y_i)}^2 dx
    \\&
    + \frac{\pbra{1-\lambda}^2}{\lambda} \int p(x)  \pbra{\sum_i p(y_i|x) \frac{d}{d\epsilon} d_\phi(x,y_i)}^2 dx
    \end{aligned}
\end{equation*}
%
which can be written more precisely as
%
\begin{equation*}
    \begin{aligned}
    0&=\int p(x) \sum_i p(y_i|x) \pder{^2 \phi(y_i)}{y_i^2} \psi^\T 
    \\&
    \sbra{ I -  \frac{1-\lambda}{\lambda} \pder{^2 \phi(y_i)}{y_i^2} (x-y_i) (x-y_i)^\T } \psi  dx
    \\&
    + \frac{1-\lambda}{\lambda} \int p(x) \pbra{\sum_i p(y_i|x) \pder{^2 \phi(y_i)}{y_i^2}(x-y_i)^\T \psi }^2 dx
    \end{aligned}
\end{equation*}
%
and finally as 
\begin{equation}
    \begin{aligned}
     0&=\sum_i p(y_i) \pder{^2 \phi(y_i)}{y_i^2} \psi^\T \sbra{ I -  \frac{1-\lambda}{\lambda} \pder{^2 \phi(y_i)}{y_i^2} C_{x|y_i} } \psi 
    \\&
    + \frac{1-\lambda}{\lambda} \int p(x) \pbra{\sum_i p(y_i|x) \pder{^2 \phi(y_i)}{y_i^2}(x-y_i)^\T \psi }^2 dx
    \end{aligned}
    \label{eq:bifurcation:finalsoc}
\end{equation}
where 
\begin{align*}
    C_{x|y_i} :&= \E{(x-y_i) (x-y_i)^\T|y_i} 
    \\&
    = \int p(x|y_i) (x-y_i) (x-y_i)^\T dx
\end{align*}
The left-hand side of (\ref{eq:bifurcation:finalsoc}) 
is positive for all perturbations $\cbra{\psi}$
if and only if the first term is positive. 
To see that, notice that the second term of (\ref{eq:bifurcation:finalsoc})
is clearly non-negative. 
For the left-hand side to be non-positive, the first term needs to 
be non-positive as well, i.e., there should exist at least one codevector value, 
say $y_n$, such that $p(y_n)>0$ and 
$\sbra{ I -  \frac{1-\lambda}{\lambda} \pder{^2 \phi(y_n)}{y_n^2} C_{x|y_n}}\preceq 0$.
In this case, there always exist a perturbation vector $\cbra{y}$
such that $y = 0$, $\forall y\neq y_n$, and $\sum_{y=y_n} \psi =0 $, 
that vanishes the second term, i.e., 
$\frac{1-\lambda}{\lambda} \int p(x) \pbra{\sum_i p(y_i|x) \pder{^2 \phi(y_i)}{y_i^2}(x-y_i)^\T \psi }^2 dx $ $= 0 $.
In other words we have shown that
\begin{align*}
    \frac{d^2}{d\epsilon^2} F^*(y) &> 0 \quad \Leftrightarrow \quad
    \exists y_n \text{  s.t.  } p(y_n)>0 
    \\&
    \text{  and  } \sbra{ I -  \frac{1-\lambda}{\lambda} \pder{^2 \phi(y_n)}{y_n^2} C_{x|y_n}}\succ 0
\end{align*}
which means that bifurcation occurs under the following condition 
\begin{equation*}
    \exists y_n \text{  s.t.  } p(y_n)>0 \text{  and  } \det\sbra{ I -  \frac{1-\lambda}{\lambda} \pder{^2 \phi(y_n)}{y_n^2} C_{x|y_n}} = 0
\end{equation*}
%
%
which completes the proof. 
%

\section{Proof of Theorem \ref{thm:ODA2timescale} (Convergence of the Two-Timescale Online Learning Rule).} 
\label{App:online-regression}

\noindent
It follows directly from the following fundamental result from 
stochastic approximation theory that is included for completeness:
\begin{theorem}[Ch. 6 of \cite{borkar1997stochastic}]
Consider the sequences $\cbra{x_n}\in S\subseteq\mathbb{R}^d$ and 
$\cbra{y_n}\in \Sigma\subseteq\mathbb{R}^k$,
generated by the iterative stochastic approximation schemes:
\begin{align}
x_{n+1} = x_n + \beta(n) \sbra{f(x_n,y_n) + M_{n+1}^{(x)}} 
\label{eq:sa_timescales_x}	\\
y_{n+1} = y_n + \alpha(n) \sbra{g(x_n,y_n) + M_{n+1}^{(y)}}
\label{eq:sa_timescales_y}
\end{align}

for $n \geq 0$ and $M_n^{(x)}$, $M_n^{(y)}$ martingale difference sequences, 
and assume that
$\sum_n \alpha(n) = \sum_n \beta(n) = \infty$, 
$\sum_n ( \alpha^2(n)+\beta^2(n) ) <\infty$, and 
$\nicefrac{\alpha(n)}{\beta(n)}\rightarrow 0$,
%
%
with the last condition implying that the iterations for $\cbra{y_n}$
run on a slower timescale than those for $\cbra{x_n}$. 
If the equation 
\begin{equation*}
\dot x (t) = f(x(t),y),\ x(0)=x_0
\end{equation*}
has an asymptotically stable equilibrium $\lambda(y)$ 
for fixed $y$ and some Lipschitz mapping $\lambda$, and the equation 
\begin{equation*}
\dot y (t) = g(\lambda(y(t)),y(t)),\ y(0)=y_0
\end{equation*}
has an asymptotically stable equilibrium $y^*$, 
then, almost surely, $(x_n,y_n)$ converges to $(\lambda(y^*),y^*)$.
\label{thm:borkar_timescales}
\end{theorem}
%

\section{Proof of Theorem \ref{thm:bayes} (Convergence to Bayes-Optimal Classifier).} 
\label{App:bayes}

\noindent
The Bayes risk for minimum probability of error is given by
\begin{align}
    J_B(\mu,c_\mu) :&= 
		\pi_1 \sum_{i:c_{\mu_i}=0} \mathbb{P}\cbra{X\in S_i | c = 1} 
  \\&
		+\pi_0 \sum_{i:c_{\mu_i}=1} \mathbb{P}\cbra{X\in S_i | c = 0} 
\end{align}
where $\mathbb{P}\cbra{X\in S_i | c = j} = \int_{S_i} p(x|c=j) ~dx$, 
and $p(x|c=j)$ represents the conditional probability density function.

Consider the observations $\cbra{(x_n,c_n)}$, and let 
$\hat p(x|c=j)$, $j=1,2$, be strongly consistent density estimators.
Then the estimated risk 
\begin{align}
    \hat J_B(\mu,c_\mu) :&= 
		\hat \pi_1 \sum_{i:c_{\mu_i}=0} \mathbb{\hat P}\cbra{X\in S_i | c = 1} 
  \\&
		+\hat \pi_0 \sum_{i:c_{\mu_i}=1} \mathbb{\hat P}\cbra{X\in S_i | c = 0} 
\end{align}
where $\mathbb{\hat P}\cbra{X\in S_i | c = j} = \int_{S_i} \hat p(x|c=j)$ and
$\hat \pi_j = \frac{\sum_n \mathds{1}_{\sbra{c_n=j}}}{n}$,
converges almost surely to $J_B$, i.e., 
\begin{equation}
    \hat J_B(\mu,c_\mu) \rightarrow J_B(\mu,c_\mu),\ a.s. 
\end{equation}
as $n\rightarrow\infty$.
This follows from the fact that 
\begin{equation}
    \hat \pi_j \hat p(x|c=j) \rightarrow \pi_j p(x|c=j) ,\ a.s.
\end{equation}
and the Lebesgue dominated convergence theorem.
Therefore, the classification rule 
$
\hat c = \argmax_ j \hat \pi_j \hat p(x|c=j)
$
converges to a Bayes-optimal classification rule.

\vspace{-3.5em}


\ifx\mycmd\undefined

\begin{IEEEbiography}[{\includegraphics[width=1in,height=1.25in,clip,keepaspectratio]
{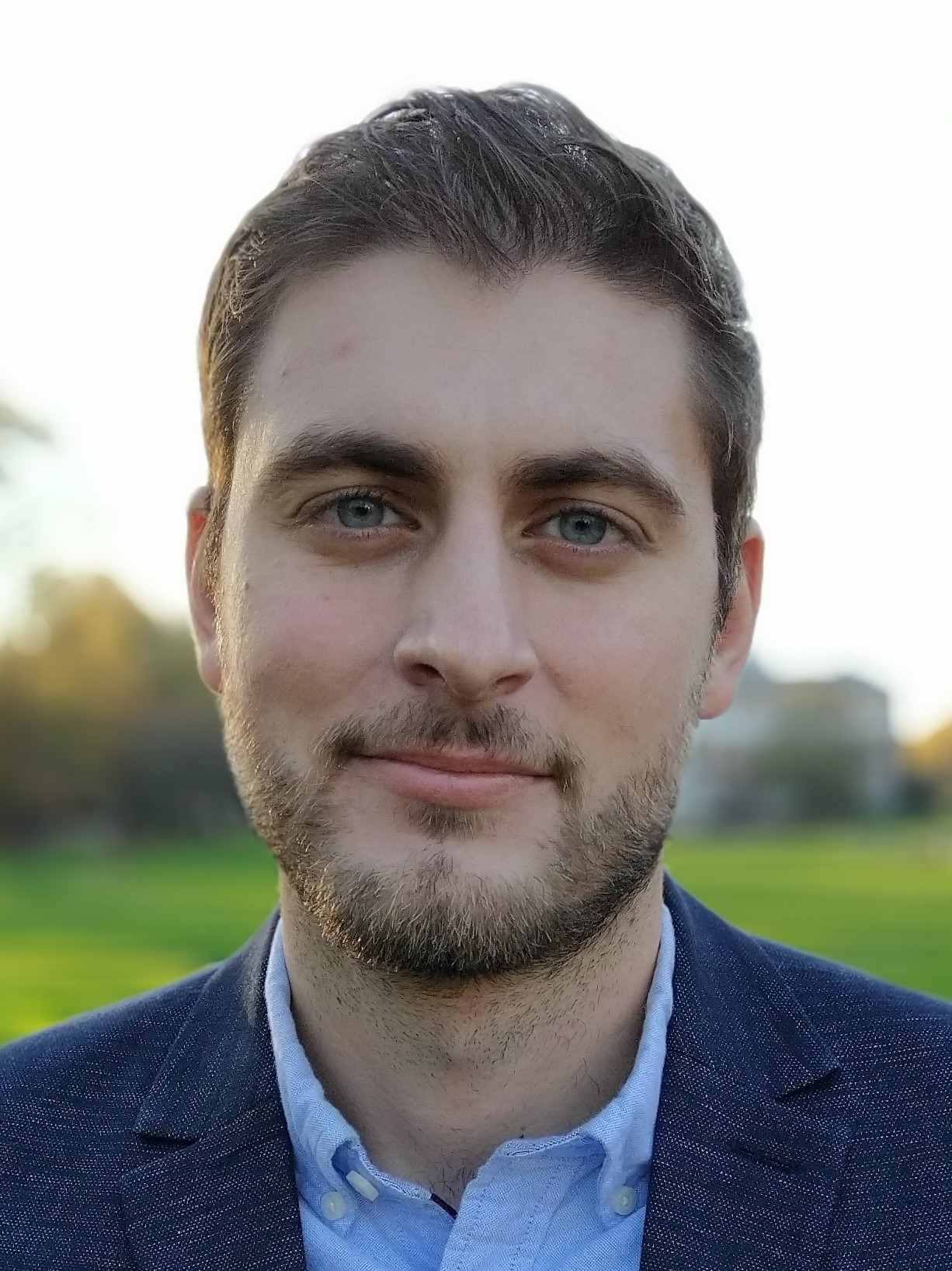}}]{Christos N. Mavridis} (M'20) 
received the Diploma degree in electrical and computer engineering from the National Technical University of Athens, Greece, in 2017,
and the M.S. and  Ph.D. degrees in electrical and computer engineering at the University of Maryland, College Park, MD, USA, in 2021. 
His research interests include learning theory, stochastic optimization, systems and control theory, multi-agent systems, and robotics. 

He is currently a postdoctoral associate at the University of Maryland, and a visiting postdoctoral fellow at KTH Royal Institute of Technology, Stockholm. He has worked as a research intern for the Math and Algorithms Research Group at Nokia Bell Labs, NJ, USA, and the System Sciences Lab at Xerox Palo Alto Research Center (PARC), CA, USA. 

Dr. Mavridis is an IEEE member, and a member of the Institute for Systems Research (ISR) and the Autonomy, Robotics and Cognition (ARC) Lab. He received the Ann G. Wylie Dissertation Fellowship in 2021, and the A. James Clark School of Engineering Distinguished Graduate Fellowship, Outstanding Graduate Research Assistant Award, and Future Faculty Fellowship, in 2017, 2020, and 2021, respectively. He has been a finalist in the Qualcomm Innovation Fellowship US, San Diego, CA, 2018, and he has received the Best Student Paper Award (1st place) in the IEEE International Conference on Intelligent Transportation Systems (ITSC), 2021.
\end{IEEEbiography}

\vspace{-3.5em}

\begin{IEEEbiography}[{\includegraphics[width=1in,height=1.25in,clip,keepaspectratio]
{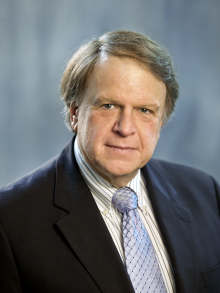}}]{John S. Baras} (LF'13) 
received the Diploma degree in electrical and mechanical engineering from the National Technical University of Athens, Athens, Greece, in 1970, and the M.S. and Ph.D. degrees in applied mathematics from Harvard University, Cambridge, MA, USA, in 1971 and 1973, respectively.

He is a Distinguished University Professor and holds the Lockheed Martin Chair in Systems Engineering, with the Department of Electrical and Computer Engineering and the Institute for Systems Research (ISR), at the University of Maryland College Park. From 1985 to 1991, he was the Founding Director of the ISR. Since 1992, he has been the Director of the Maryland Center for Hybrid Networks (HYNET), which he co-founded. His research interests include systems and control, optimization, communication networks, applied mathematics, machine learning, artificial intelligence, signal processing, robotics, computing systems, security, trust, systems biology, healthcare systems, model-based systems engineering.

Dr. Baras is a Fellow of IEEE (Life), SIAM, AAAS, NAI, IFAC, AMS, AIAA, Member of the National Academy of Inventors and a Foreign Member of the Royal Swedish Academy of Engineering Sciences. Major honors include the 1980 George Axelby Award from the IEEE Control Systems Society, the 2006 Leonard Abraham Prize from the IEEE Communications Society, the 2017 IEEE Simon Ramo Medal, the 2017 AACC Richard E. Bellman Control Heritage Award, the 2018 AIAA Aerospace Communications Award. In 2016 he was inducted in the A. J. Clark School of Engineering Innovation Hall of Fame. In 2018 he was awarded a Doctorate Honoris Causa by his alma mater the National Technical University of Athens, Greece.   
\end{IEEEbiography}

\fi

\end{document}

%% file: macros.tex
\usepackage{amsmath,amssymb,amsfonts,mathtools} 
\usepackage{dsfont,eucal,bbm,bm,nicefrac} 
\usepackage{graphicx,float,subcaption,booktabs} 
	\graphicspath{{./figures/}}
\usepackage{algorithmic} 
\usepackage{tikz,xcolor} 
\usepackage{cite}
\usepackage{textcomp}
\usepackage{epstopdf}
\ifpdf
  \DeclareGraphicsExtensions{.eps,.pdf,.png,.jpg}
\else
  \DeclareGraphicsExtensions{.eps}
\fi

\newsiamremark{remark}{Remark}
\newsiamremark{hypothesis}{Hypothesis}
\crefname{hypothesis}{Hypothesis}{Hypotheses}
\newsiamthm{claim}{Claim}
\newsiamthm{problem}{Problem}
\newsiamthm{example}{Example}
\newcommand{\T}{\ensuremath{\mathrm{T}}} 

\newcommand{\norm}[1]{\ensuremath{\left\| #1\right\|}}
\newcommand{\pbra}[1]{\ensuremath{\left( #1\right)}}
\newcommand{\sbra}[1]{\ensuremath{\left[ #1\right]}}
\newcommand{\cbra}[1]{\ensuremath{\left\{ #1\right\}}}
\newcommand{\abra}[1]{\ensuremath{\left< #1\right>}}

\newcommand{\pder}[2]{\ensuremath{\frac{\partial #1}{\partial #2}}}
\newcommand{\E}[1]{\ensuremath{\mathbb{E}\left[ #1\right]}}

\DeclareMathOperator*{\argmax}{arg\,max}
\DeclareMathOperator*{\argmin}{arg\,min}

\usepackage{amsopn}

\headers{Multi-Resolution Online Deterministic Annealing}{C. N. Mavridis, J. S. Baras}

\title{Multi-Resolution Online Deterministic Annealing: \\A Hierarchical and Progressive Learning Architecture\thanks{Under review.
\funding{This work was partially supported by the 
Defense Advanced Research Projects
Agency (DARPA) under Agreement No. HR00111990027, 
by ONR grant N00014-17-1-2622, 
and by a grant from Northrop Grumman Corporation.}}}

\author{Christos N. Mavridis 
\and John S. Baras%
\thanks{Department of Electrical and Computer Engineering and 
the Institute for Systems Research, 
University of Maryland, College Park, USA
  (\email{mavridis@umd.edu}, \email{baras@umd.edu}).}
  }